\documentclass[letterpaper]{article} 
\usepackage{aaai24}  
\usepackage{times}  
\usepackage{helvet}  
\usepackage{courier}  
\usepackage[hyphens]{url}  
\usepackage{graphicx} 
\usepackage{natbib}  
\usepackage{caption} 
\usepackage{algorithm}
\usepackage{algorithmic}

\usepackage{newfloat}
\usepackage{listings}
\DeclareCaptionStyle{ruled}{labelfont=normalfont,labelsep=colon,strut=off} 
\floatstyle{ruled}
\newfloat{listing}{tb}{lst}{}
\floatname{listing}{Listing}
\usepackage{url}            %
\usepackage{booktabs}       %
\usepackage{amsfonts}       %
\usepackage{amsmath}
\usepackage{amssymb}
\usepackage{amsthm}         %
\usepackage{algorithm}
\usepackage{algorithmic}
\usepackage{nicefrac}       %
\usepackage{microtype}  
\usepackage{makecell}    %
\usepackage {multirow}
\usepackage[figuresright]{rotating}

\usepackage{diagbox}
\usepackage{enumitem}

\newcommand{\E}{\mathbb{E}}
\newcommand{\Esub}[1]{\underset{#1}{\E}}
\newcommand{\Ei}{\bar{\mathbb{E}}}
\newcommand{\Eisub}[1]{\underset{#1}{\Ei}}
\newcommand{\D}{{\mathcal{D}^{env}}}

\newcommand{\hatT}{\widehat{T}}
\newcommand{\hatM}{\widehat{M}}
\newcommand{\tT}{T}
\newcommand{\tM}{M}

\newcommand{\cS}{{\mathcal{S}}}
\newcommand{\cA}{\mathcal{A}}
\newcommand{\cT}{\mathcal{T}}
\newcommand{\cO}{\mathcal{O}}
\newcommand{\ccO}{\tilde{ \mathcal{O}}}

\newcommand{\pM}{M^p}
\newcommand{\tilM}{M^p}

\newcommand{\ppD}{{\mathcal{D}^p_{\pi^p}}}
\newcommand{\poD}{{\mathcal{D}^p_{\pi^o}}}
\newcommand{\ooD}{{\mathcal{D}^o_{\pi^o}}}
\newcommand{\Ppi}{\pi^p}

\newcommand{\oM}{M^o}
\newcommand{\oD}{{\mathcal{D}^o_{\pi^o}}}

\newcommand{\Opi}{\pi^o}
\newcommand{\cut}[1]{}

\DeclareMathOperator*{\argmax}{arg\,max}

\usepackage{mathtools}
\usepackage{microtype}
\usepackage{graphicx}
\usepackage{subfigure}
\usepackage{booktabs} 

\usepackage{colortbl}
\usepackage{xcolor}

\theoremstyle{definition}
\newtheorem{definition}{Definition}
\newtheorem{proposition}{Proposition}
\newtheorem*{proposition*}{Proposition} 
\newtheorem{assumption}{Assumption}
\newtheorem{lemma}{Lemma}
\newtheorem{theorem}{Theorem}

\title{Optimistic Model Rollouts for Pessimistic Offline Policy Optimization}
\author{
    Yuanzhao Zhai\textsuperscript{\rm 1,2},
    Yiying Li\textsuperscript{\rm 3},
    Zijian Gao\textsuperscript{\rm 1,2},
    Xudong Gong\textsuperscript{\rm 1,2}, \\
    Kele Xu\textsuperscript{\rm 1,2},
    Dawei Feng\textsuperscript{\rm 1,2}\thanks{Corresponding author.},\
    Ding Bo\textsuperscript{\rm 1,2},
    Huaimin Wang\textsuperscript{\rm 1,2} \
}
\affiliations{
    \textsuperscript{\rm 1}National University of Defense Technology, Changsha, China 
    \textsuperscript{\rm 2}State Key Laboratory of \\ Complex \& Critical Software Environment
    \textsuperscript{\rm 3}Artificial Intelligence Research Center, DII, Beijing, China \\
     \{yuanzhaozhai, liyiying10, gaozijian19, gongxudong09, dingbo, hmwang\}@nudt.edu.cn, \\ xukelele@163.com, davyfeng.c@qq.com

}

\begin{document}

\maketitle

\begin{abstract}
Model-based offline reinforcement learning (RL) has made remarkable progress, offering a promising avenue for improving generalization with synthetic model rollouts. 
Existing works primarily focus on incorporating pessimism for policy optimization, usually via constructing a Pessimistic Markov Decision Process (P-MDP). 
However, the P-MDP discourages the policies from learning in out-of-distribution (OOD) regions beyond the support of offline datasets, which can under-utilize the generalization ability of dynamics models.
In contrast, we propose constructing an Optimistic MDP (O-MDP).
 We initially observed the potential benefits of optimism brought by encouraging more OOD rollouts. 
Motivated by this observation, we present ORPO, a simple yet effective model-based offline RL framework. ORPO generates Optimistic model Rollouts for Pessimistic offline policy Optimization. 
Specifically, we train an optimistic rollout policy in the O-MDP to sample more OOD model rollouts.
Then we relabel the sampled state-action pairs with penalized rewards and optimize the output policy in the P-MDP.
Theoretically, we demonstrate that the performance of policies trained with ORPO can be lower-bounded in linear MDPs.
Experimental results show that our framework significantly outperforms P-MDP baselines by a margin of 30\%, achieving state-of-the-art performance on the widely-used benchmark.
Moreover, ORPO exhibits notable advantages in problems that require generalization.
\end{abstract}

\section{Introduction}

In scenarios where online trial-and-error are too costly or prohibited, such as autonomous driving~\cite{yu2018bdd100k}, healthcare~\cite{gottesman2019guidelines}, and robotics~\cite{mandlekar2020iris}, offline RL~\cite{levine2020offline} has emerged as a solution to leverage previously-collected datasets.
While successful, recent research~\cite{ROMI, lee2022offline} demonstrates that model-free offline RL methods typically learn overly conservative policies and lack generalization beyond the datasets.




%
Model-based approach, which leverages a learned dynamics model to generate rollouts for policy optimization~\cite{sutton1990integrated}, has been introduced to offline RL, achieving remarkable progress~\cite{MOPO, RAMBO}.
%
In the context of offline RL, the dynamic models may exhibit inaccuracies due to limited datasets.
To avoid over-estimation on out-of-distribution (OOD) data, prior methods construct a Pessimistic Markov Decision Process (P-MDP) based on uncertainty quantification of dynamics models~\cite{MOPO, MOReL}, which lower-bounds the real MDP.


Dynamics models trained in a supervised manner can exhibit refined generalization capacity for some near-distribution OOD state-action pairs, which is mainly studied and utilized in the model-based online RL ~\cite{MBPO, moerland2023model}.
Recent works~\cite{EDAC,PBRL} have demonstrated that OOD sampling can effectively regularize behaviors and enhance generalization for offline policy optimization.
However, P-MDP used in prior model-based offline RL penalizes OOD state-action pairs that have high uncertainty (Figure~\ref{fig-Dyna}), discouraging policies from sampling in OOD regions. 
Therefore, utilizing only pessimistic rollouts may under-utilize the dynamics model, thereby limiting generalization.
We verify this by designing a toy task in Figure~\ref{fig:toy}.

\begin{figure}[t]
	\label{ORPO}
	\begin{center}
		\subfigure[Prior approach.
		\label{fig-Dyna}]
		{
			\centering
			\includegraphics[width=0.41\linewidth]{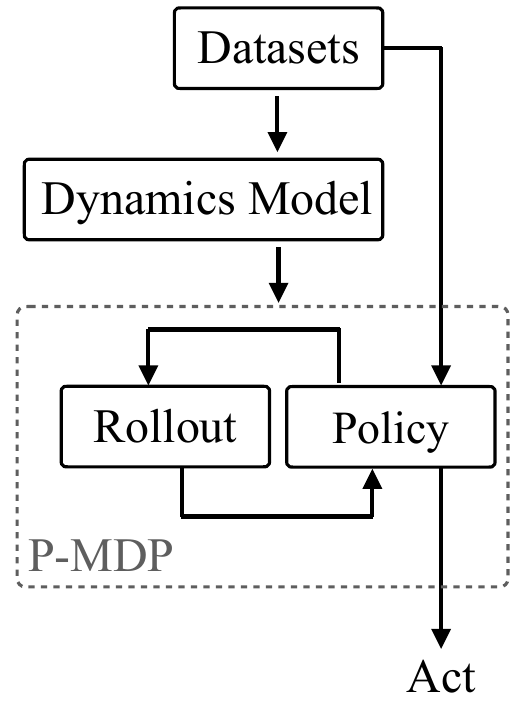}
		}
		\subfigure[ORPO framework.
		\label{fig-ORPO}]
		{
			\centering
			\includegraphics[width=0.49\linewidth]{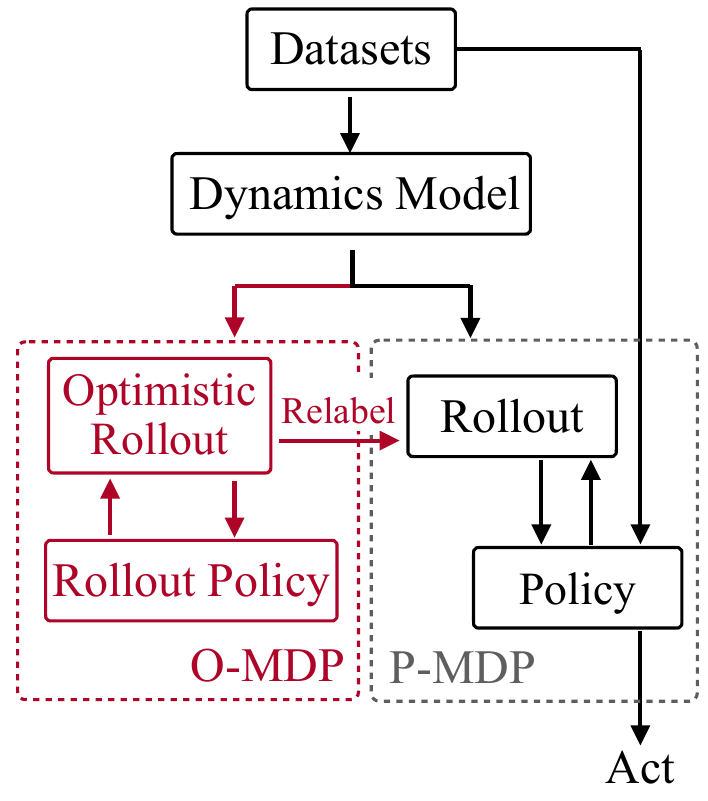}
		}
		\caption{
			(a) Previous model-based offline RL generates model rollouts and optimizes the policy within the P-MDP.
			(b) 
			We decouple the training of optimistic rollout policies from the pessimistic policy optimization.
		}
	\end{center}
\end{figure}


\begin{figure*}[!ht]
		\subfigure[Scenario of our toy experiment.
		\label{fig-toy-scene}]
		{
			\centering
			\includegraphics[width=0.213\linewidth]{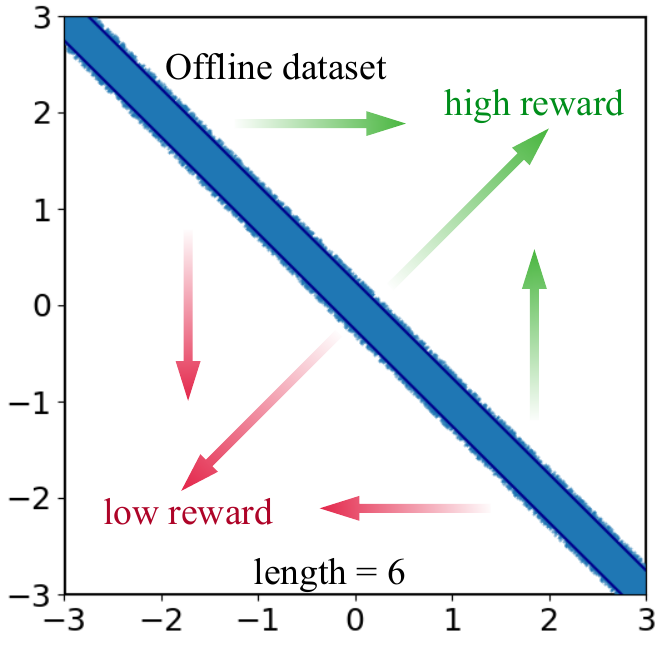}
		}
	\quad
		\subfigure[Uncertainty value of the learned dynamics model.
		\label{fig-toy-uncertainty}]
		{
			\centering
			\includegraphics[width=0.25\linewidth]{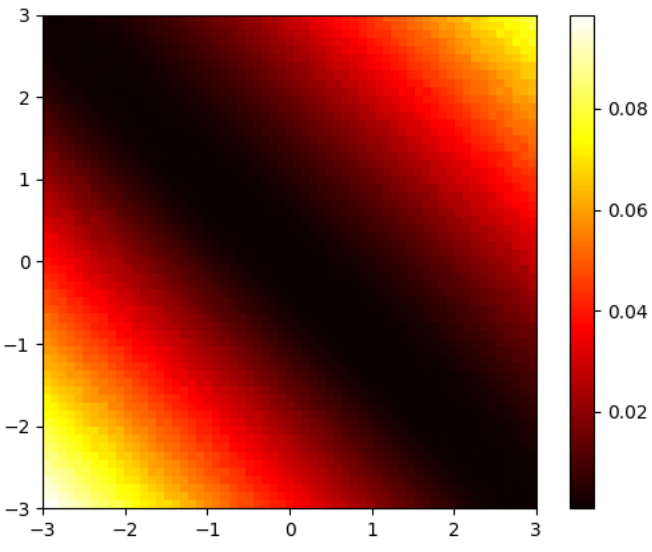}
		}
	\quad
		\subfigure[Evaluation trajectories of the policy trained in the P-MDP.
		\label{fig-toy-mopo}]
		{
			\centering
			\includegraphics[width=0.213\linewidth]{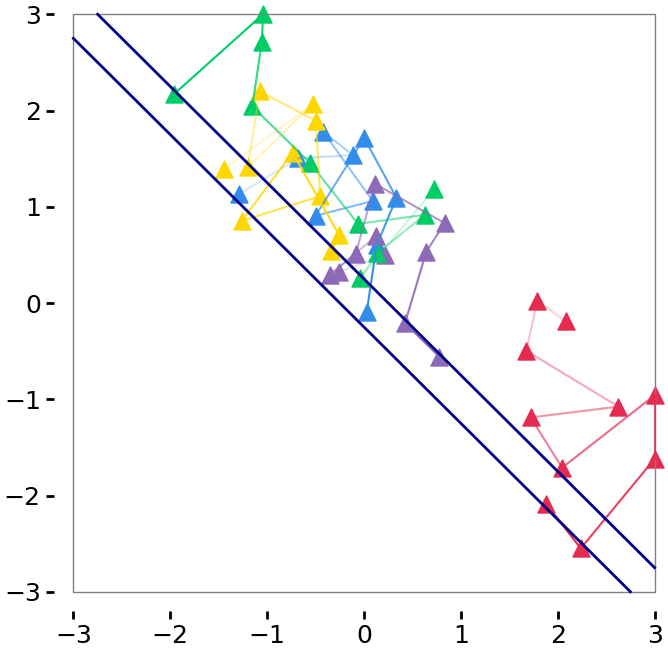}
		}
	\quad
		\subfigure[Evaluation trajectories of the policy trained with ORPO.
		\label{fig-toy-orpo}]
		{
			\centering
			\includegraphics[width=0.213\linewidth]{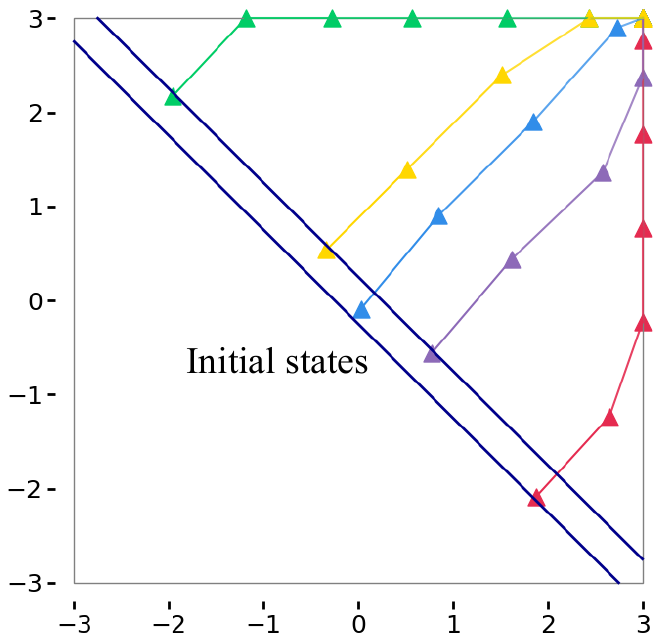}
		}
		\caption{
			(a) In toy experiments with a 2-dimensional continuous state space and action space, the coordinate origin (0, 0) is taken as the central point of the square region. 
			The agent starts at the region between lines $y=-x-0.25$ and $y=-x+0.25$, and the goal is to move upper right to obtain high rewards.
			The offline dataset only contains transitions whose state is in the initial area.
			(b) The further the states are from the offline dataset, the higher the estimated uncertainty value by the dynamics model.
			(c) The policy trained with MOPO~\cite{MOPO} with only P-MDP can not reach regions with high reward but high uncertainty.
			(d) With more optimistic model rollouts but optimization in the same P-MDP, ORPO agents can learn to reach states with high rewards and avoid regions with low rewards.
			Please refer to Appendix~\ref{exp-toy} for the detailed experimental setup.
		}\label{fig:toy}
\end{figure*}

This paper delves into the efficacy of optimism in the context of model-based offline RL, aiming to take full advantage of the learned dynamics model.
To introduce optimism when generating model rollouts, we can flip the sign of uncertainty penalties of the Pessimistic MDP (P-MDP), resulting in an Optimistic MDP (O-MDP). 
However, using only O-MDP in model-based offline RL contrasts the provably efficient pessimism~\cite{pevi-2021} in offline RL, and may mislead policies to risky state-action regions with large dynamics model errors.
Hence, the central question that this work is trying to answer is: can we train an offline policy that exploits the generalization ability of dynamics models while still adopting provably efficient pessimism?

To this end, we present a novel model-based offline RL algorithmic framework called ORPO (Figure~\ref{fig-ORPO}), which decouples the training of rollout policies from the pessimistic policy optimization.
Specifically, we construct an O-MDP to train optimistic rollout policies, which have a higher probability of accessing OOD state-action regions based on the generalization ability of dynamics models.
Subsequently, we relabel the optimistic model rollouts by assigning them penalized rewards in the P-MDP.
The agent trained with the relabeled optimistic rollouts is more likely to select OOD actions when their value estimations are high, while avoiding low-value risky regions, as shown in Figure~\ref{fig-toy-orpo}. 
In summary, our main contributions are:

\begin{itemize}[leftmargin=*]
	\setlength\itemsep{0em}
	\item We introduce the construction of an O-MDP in the model-based offline RL framework, highlighting its potential benefits derived from encouraging increased OOD sampling.
	\item  We present ORPO, a novel framework that generates optimistic model rollouts for pessimistic offline policy optimization. We theoretically provide the lower bound of the expected return of policies trained with ORPO.
	\item Through empirical evaluations, ORPO policies outperform the P-MDP baseline by a substantial margin of 30\%, and achieve competitive or superior scores compared to baseline methods in 8 out of 12 datasets from the D4RL benchmark~\cite{fu2020d4rl}. Furthermore, our method has better performance compared to the state-of-the-art in two datasets requiring policy to generalize.
\end{itemize}

	\section{Related Works}
Off-policy online RL algorithms ~\cite{TD3, SAC} often suffer from inefficiency due to extrapolation errors~\cite{BCQ}. These errors arise from overestimating the values of out-of-distribution (OOD) state-action pairs beyond the support of offline datasets.
%
Offline RL is proposed to learn effective policies from a logged dataset without interacting with the environment~\cite{levine2020offline}, which can generally be categorized into two types: model-free and model-based.
Model-free offline RL methods learn conservative value functions~\cite{CQL, IQL} or directly constrain the policy~\cite{BCQ, BEAR,TD3BC} to preclude OOD actions.
However, policy trained by such methods may be overly conservative~\cite{lee2022offline}, lacking generalization ability beyond the offline dataset~\cite{ROMI}.


\subsubsection{Model-based Offline RL. }
Model-based offline RL algorithms first train a dynamics model using supervised learning with the logged dataset.
Then the dynamics model can be used to optimize policies, in which Dyna-style algorithm~\cite{sutton1990integrated} is adopted by a number of recent methods~\cite{MOPO,COMBO,clavera2019model,rafailov2021offline}.
By utilizing the additional synthetic data generated by the learned dynamics model, model-based offline RL methods have the potential to exhibit better generalization abilities compared to model-free \cite{CQL, IQL, TD3BC}.
%
%
%
%

Since the limitation of the logged dataset, it is essential to quantify how trustable the model is for specific rollouts.
Both MOPO \cite{MOPO} and MOReL \cite{MOReL} construct the P-MDP to optimize the policy, where rewards are penalized according to uncertainty quantification.
Many recent works aim to incorporate pessimism into policy optimization, via backward dynamics model \cite{ROMI, CABI}, uncertainty-free conservatism  \cite{COMBO} or robust MDPs \cite{guo2022model, RAMBO}.
In contrast, we investigate the potential benefits of optimism for training rollout policies. We adopt the P-MDP from MOPO for pessimistic policy optimization and introduce the O-MDP for generating model rollouts. Importantly, our proposed framework is not limited to MOPO and can be easily combined with other model-based offline RL methods.

\subsubsection{Uncertainty Aware Reinforcement Learning.}
Uncertainty plays a crucial role in RL.
Optimism in the Face of Uncertainty (OFU)~\cite{bandit-2011} principle is commonly employed in online RL for active and efficient environment exploration~\cite{lockwood2022review}, which is provably efficient \cite{bandit-2011, lsvi-2020}.
Uncertainty is also widely used in model-based online RL for controlling the model usage~\cite{SLBO, MBPO, pan2020trust}. 

In offline RL, uncertainty is typically utilized for pessimism.
As aforementioned, certain model-based offline RL methods~\cite{lu2022revisiting} estimate the uncertainty of the dynamics model to construct P-MDPs. Additionally, recent model-free methods~\cite{EDAC,PBRL,UWAC} employ the uncertainty quantification of Q-functions to penalize OOD state-action pairs.
Our proposed framework is closely related to both the provably efficient designs for exploration in online RL and pessimism in offline RL.


\section{Preliminaries}
We define a Markov Decision Process (MDP) as the tuple $\tM = (\mathcal{S}, \mathcal{A}, \tT, r, \gamma)$, where $\mathcal{S}$ and $\mathcal{A}$ denote the state and action space, $ \tT(s'|s, a)$ represents the dynamics or transition distribution, $r(s, a)$ is the reward function, and $\gamma \in (0, 1)$ is the discount factor.
Let $\mathbb{P}^\pi_{\tT,t}(s)$ denote the probability of being in state $s$ at time step $t$ if actions are sampled according to $\pi$ and transitions according to $\tT$.
Let $\rho^\pi_{ \tT}(s,a)$ be the discounted occupancy measure of policy $\pi$ under dynamics $\tT$:
$\rho^\pi_{ \tT}(s,a) := \sum_{t=0}^\infty \gamma^t \mathbb{P}^\pi_{ \tT,t}(s) \pi(a | s) $. 
The goal is to find a policy $\pi(a|s)$ that maximizes the expected discounted return  $\eta_{\tM}(\pi) = \Eisub{(s,a)\sim \rho^\pi_{\tT}}[r(s,a)]$. 
The value function $V_{\tM}(s) := \Esub{\pi, \tT}\left[\sum_{t=0}^\infty \gamma^t r(s_t, a_t) | s_0=s\right]$ gives the expected discounted return when starting from state $s$.


For offline RL where agents can not interact with the environment, we have a previously-collected static dataset $\D = \{(s_{j}, a_{j}, r_{j}, s_{j+1})\}_{j=1}^{J}$, which consists of $J$ transition tuples from trajectories collected by a behavior policy.
Canonical model-based offline RL methods typically train an ensemble of $N$ probabilistic networks as the dynamics model $\widehat T = \{\widehat{T}_i(\hat s' |s,a)  = \mathcal{N}(\mu_i (s,a) , \Sigma_i (s,a) )\}_{i=1}^N$ 
to predict the next state $s'$ from a state-action pair.
Following previous works~\cite{MOPO, MOReL}, we assume the reward function r is known. If $r(\cdot)$ is unknown, it can also be learned from data.
The learned dynamics model $\hatT$ define a model MDP $\hatM = (\mathcal{S}, \mathcal{A}, \hatT, r, \gamma)$.
Then the goal switches to find a policy $\pi(a|s)$ that maximizes the expected discounted return with respect to $\rho^\pi_{ \hatT}$, as in $\eta_{\hatM}(\pi) = \Eisub{(s,a)\sim \rho^\pi_{\hatT}}[r(s,a)]$.

%
%

Model-based offline RL methods often construct a P-MDP for pessimistic offline policy optimization.
Notably, based on the model error between the true and learned dynamics, 
\begin{equation}
	\label{equation:G}
	G_{\hatM}(s,a) := \Esub{s' \sim \hatT(s,a)}[V_M(s')] - \Esub{s' \sim \tT(s,a)}[V_M(s')],
\end{equation}
MOPO assumes that there is an admissible model uncertainty $u(s,a)$ that can upper-bound the model error $|G^\pi_{\hatM}(s,a)|$:
\begin{equation}
	\label{equ:condition}
	u(s,a)	\ge |G^\pi_{\hatM}(s,a)|,  \forall s \in \mathcal{S}, a \in \mathcal{A}.
\end{equation}
Penalized by the estimator, pessimistic reward $r^p(s,a) =  r(s,a) - \lambda^p u(s,a)$ can be used to construct P-MDP as $\pM = (\mathcal{S}, \mathcal{A}, \hatT,  r^p, \gamma)$, where $\lambda^p:=\gamma c$ denotes the degree of pessimism and $c$ is a constant.
Define the average model uncertainty $\epsilon_u(\pi)$ as:
\begin{align}
	\epsilon_u(\pi) := \Eisub{(s,a)\sim \rho^\pi_{\hatT}}u(s,a)\label{eqn:11}.
\end{align}
Then the lower bound of performance in the real MDP can be established by the P-MDP~\cite{MOPO} as:
\begin{align}
	\label{equ:lower-bound}
	\eta_{\tM}(\pi) = 
	& \Eisub{(s,a) \sim \rho^\pi_{\hatT}}\left[ r(s,a) - \gamma |G^\pi_{\hatM}(s,a)|\right] \nonumber 		\\ \ge &\Eisub{(s,a) \sim \rho^\pi_{\hatT}}\left[ r(s,a) - \lambda^p u(s,a) \right] \nonumber \\ = &\eta_{\hatM}(\pi)-\lambda^p \epsilon_u(\pi) = \eta_{\pM}(\pi).
\end{align}
According to Equation \ref{equ:lower-bound}, we can optimize policies in the real MDP by improving their performance in the P-MDP.


\section{Proposed Method}
In this section, we present the construction of an O-MDP and discuss its potential benefits in Section~\ref{sec:O-MDP}.
Next, we provide an overview of the ORPO framework in Section~\ref{sec:framework}.
To establish a solid theoretical foundation for ORPO, we delve into the theoretical analysis in Section~\ref{sec:theoretical}.

\subsection{Optimistic MDP Construction}
\label{sec:O-MDP}
To introduce optimism when generating model rollouts, we flip the sign of the uncertainty penalty in the Pessimistic MDP to construct an Optimistic MDP (O-MDP) $\oM = (\mathcal{S}, \mathcal{A}, \hatT, r^o, \gamma)$, where $r^o(s,a) =   r(s,a) + \lambda^o u(s,a) $.
Subsequently, we train the rollout policy within the O-MDP, optimizing the following objective (Equation~\ref{equ:objective}), where $\lambda^o$ serves as a coefficient to regulate the level of optimism.
\begin{equation}
	\label{equ:objective}
	\Opi = \argmax_{\pi} \Eisub{(s,a) \sim \rho^\pi_{\hatT}}[   r(s,a) + \lambda^o u(s,a)].
\end{equation}
To analyze the impact of the O-MDP, we begin by comparing the model rollouts generated under the P-MDP and O-MDP settings. 
We measure the distance between the rollout actions and the offline dataset actions using the $\ell_2$ norm, calculated as $\Esub{(s,a)\sim \D} [\Vert\pi(\cdot|s)-a\Vert_2]$, where $\pi$ represents the rollout policies.
As shown in Figure~\ref{fig-distance-OMDP}, we observe that optimistic rollouts exhibit larger distances from the offline dataset compared to pessimistic rollouts, indicating that optimistic rollouts involve more OOD sampling.
While prior works~\cite{CQL, PBRL} typically sample OOD actions using random policies, we contend that OOD model rollouts generated within the O-MDP framework possess greater value. 
On the one hand, optimistic rollout policies guided by the objective in Equation~\ref{equ:objective} selectively sample actions with high uncertainty and high estimated values, as opposed to random policies. This targeted sampling strategy can lead to more informative OOD actions.
On the other hand, dynamics models with generalization capacity can generate better characterization for OOD model rollouts.

\begin{figure}[h]
	\centering
	\includegraphics[width=0.48\textwidth]{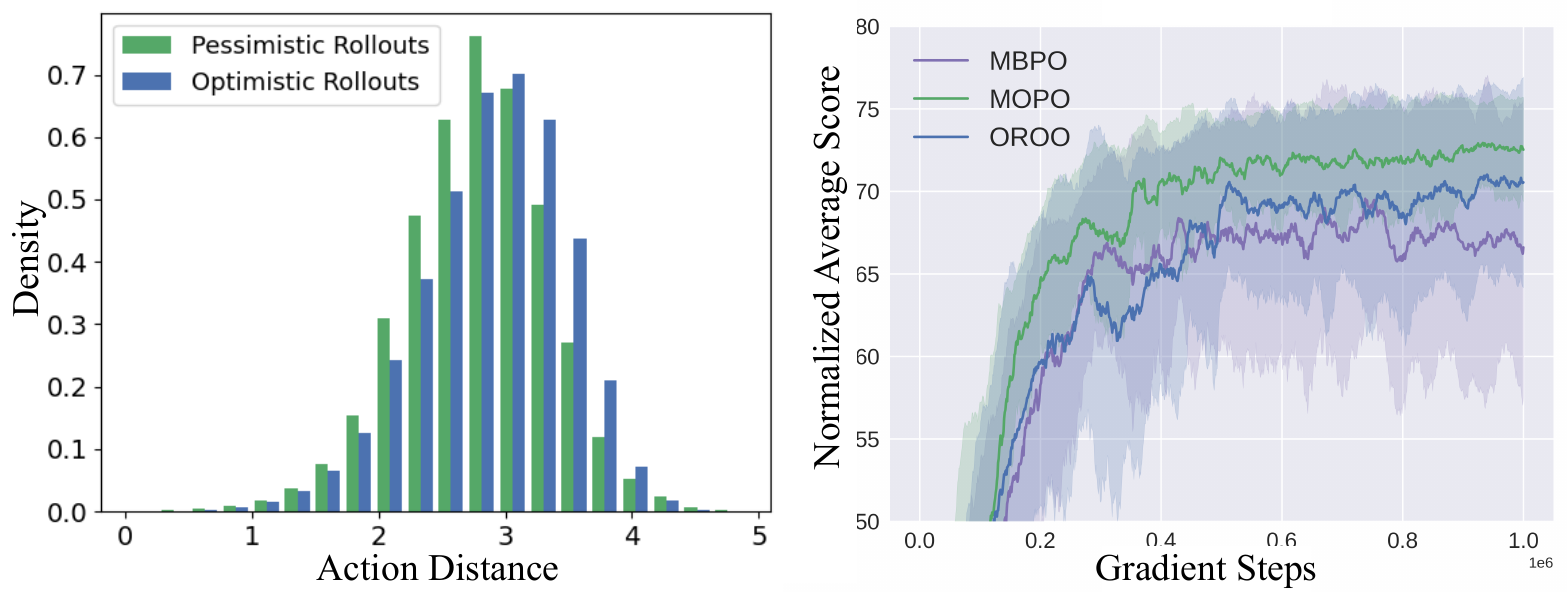}
	\caption{\small A case study of methods using model MDP (MBPO), P-MDP (MOPO), and O-MDP (OROO) on ``Halfcheetah-medium-v2'' datasets over 5 different seeds. \textbf{Left: }Histograms of distances between actions from different model rollouts and the offline dataset.
		\textbf{Right: }Learning curves of different methods.}
	\label{fig-distance-OMDP}
\end{figure}

%
%

To showcase the effectiveness of OOD model rollouts, we introduce a simple baseline called Optimistic model Rollouts for Optimistic policy Optimization (OROO), which is derived from MOPO by replacing the P-MDP with the O-MDP.
We make the intriguing observation that OROO utilizing the O-MDP, surpasses the performance of the model MDP baseline, MBPO~\cite{MBPO}.
Interestingly, our experiments in Section~\ref{exp:ablation} reveal that OROO can even outperform MOPO on certain offline RL datasets. This finding highlights the additional benefits brought by O-MDP and the potential of optimism in model-based offline RL.

However, it is important to acknowledge that optimizing policies in the O-MDP can result in a larger model error $\epsilon_u(\pi^o)$, which in turn reduces the performance lower bound as depicted in Equation~\ref{equ:lower-bound}.
This is consistent with the previous conclusion that pessimism is provably efficient in the offline setting~\cite{pevi-2021}.
In fact, the choice between P-MDP and O-MDP involves a trade-off between the generalization capacity of the dynamics model and the introduced model error.
Building upon the aforementioned analysis, our objective is to train optimistic rollout policies that encourage more OOD sampling while still utilizing the P-MDP to control the model error within an admissible range.

\subsection{Algorithmic Framework}
\label{sec:framework}

We now present our framework, ORPO, which is designed to generate optimistic model rollouts for offline policy optimization in a pessimistic manner.
In ORPO, we decouple the training of the rollout policy from the pessimistic optimization of the output policy $\Ppi$. Instead, we focus on learning a more optimistic rollout policy denoted as $\Opi$, which is optimized under the O-MDP constructed in Section~\ref{sec:O-MDP}.

\begin{algorithm}
	\caption{Framework for Optimistic Rollout for Pessimistic Policy Optimization (ORPO)}
	\label{broad-alg}
	\begin{algorithmic}[1]
		\STATE {\bfseries Require:} Offline dataset $\D$, initialized rollout policy $\Opi$ and output policy $\Ppi$.
		\STATE Train the dynamics model $\widehat{T}$ with uncertainty quantifier $u(s,a)$.
		\STATE Initialize the replay buffers $\ooD \leftarrow \varnothing,\poD \leftarrow \varnothing,\ppD \leftarrow \varnothing$.
		\FOR{epoch $1, 2, \dots$}
		\STATE Run any online RL algorithm in $\oM$ to optimize rollout policy $\Opi$, and add the rollouts in replay buffer to $\ooD$.
		\STATE Relabel $\ooD$ with penalized rewards according to P-MDP, obtaining $\poD$.
		\STATE Collect model rollouts by sampling from $\Ppi$ in $\pM$ starting from states in $\D$, and add the rollouts to $\ppD$.
		\STATE 
		Run any offline RL algorithm on $\D \cup \poD \cup \ppD$ to optimize policy $\Ppi$.
		\ENDFOR
		\STATE {\bfseries Return:} Optimized output policy $\Ppi$.
	\end{algorithmic}
\end{algorithm}

The optimistic rollout policy $\Opi$ is capable of interacting with the dynamics model, allowing us to optimize it using online RL algorithms.
During the training of $\Opi$, we collect and store the optimistic rollouts $(s, \Opi(a|s), r^o, \hat s')$ in a buffer denoted as $\ooD$.
Then we directly relabel the optimistic rollouts using the penalized reward according to the P-MDP. This relabeling process transforms the rollouts into $(s, \Opi(a|s), r^p, \hat s')$.
We then store these relabeled optimistic rollouts into another buffer $\poD$ to be used for pessimistic policy optimization.
Besides, we also store pessimistic rollouts $(s, \Ppi(a|s), r^p, \hat s')$ which are sampled by the output policy in P-MDP, denoted as $\ppD$.
Note that previous model-based offline RL methods \cite{lu2022revisiting} typically utilize $\ppD$ and $\D$ for pessimistic policy optimization.
In our framework, with the inclusion of the rollout policy $\Opi$, we can leverage the additional dataset $\poD$ to introduce more OOD state-action pairs.
Given the datasets $\ppD$, $\poD$, and the offline dataset $\D$, our objective is to derive a policy $\Ppi$ that maximizes the expected discounted return in the real MDP, i.e.,
\begin{equation}
	\Ppi = \argmax_{\pi} [ \eta_{\tM}( \pi) ].
\end{equation}
The behavior policy used to collect our synthetic dataset, which includes $\poD$, $\ppD$, and $\D$, differs significantly from the desired output policy $\Ppi$.
Therefore, we employ offline RL algorithms for pessimistic optimization.
The training of the rollout policy and pessimistic policy optimization is conducted iteratively in an alternating fashion.
The overall framework of ORPO is outlined in Algorithm \ref{broad-alg}.
For practical implementation details, please refer to Appendix \ref{appendix:practical}.



\begin{table*}[t]\centering
	\addtolength{\tabcolsep}{-3pt}
	\renewcommand\arraystretch{0.8}
	\small
	\centering
		\begin{tabular}{c@{\hspace{3pt}}l@{\hspace{-0.5pt}}r@{\hspace{-1pt}}lr@{\hspace{-0.5pt}}lr@{\hspace{-0.5pt}}lr@{\hspace{-0.5pt}}lr@{\hspace{-0.5pt}}lr@{\hspace{-0.5pt}}lr@{\hspace{-0.5pt}}lr@{\hspace{-0.5pt}}lr@{\hspace{-0.5pt}}lr@{\hspace{-0.5pt}}lr}
			\toprule
			\multicolumn{1}{c}{}     &                               & \multicolumn{8}{c|}{Model-free methods}  & \multicolumn{8}{c}{Model-based methods}  & \\
			\midrule
			\multicolumn{1}{l}{}     &                                      & \multicolumn{2}{c}{\makecell[c]{2020 \\ NeurIPS \\ CQL }}  & \multicolumn{2}{c}{\makecell[c]{2021 \\ NeurIPS \\ TD3+BC }}   &
			\multicolumn{2}{c}{\makecell[c]{2022 \\ ICLR \\ IQL }}   & \multicolumn{2}{c}{\makecell[c]{2022 \\ ICLR \\ PBRL }}& \multicolumn{2}{c}{\makecell[c]{2020 \\ NeurIPS \\ MOPO }}   & \multicolumn{2}{c}{\makecell[c]{2020 \\ NeurIPS \\ MOReL }}   & \multicolumn{2}{c}{\makecell[c]{2021 \\ NeurIPS \\ COMBO}}     &
			\multicolumn{2}{c}{\makecell[c]{2022 \\ NeurIPS \\ RAMBO}}     &
			\multicolumn{2}{c}{\makecell[c]{2022 \\ NeurIPS \\ CABI }}&
			\multicolumn{2}{c}{\makecell[c]{(Ours)\\ORPO }}\\
			\midrule
			\multirow{3}{*}{\rotatebox[origin=c]{90}{ Random}} & HalfCheetah \ \  \ \    & {27.0} &{$\pm$0.6}& {11.3} & {$\pm$0.5}  &{7.8} &{$\pm$0.3}&\multicolumn{2}{c}{ {11.0}} & {20.7}  & {$\pm$1.8}   &\multicolumn{2}{c}{25.6}&\multicolumn{2}{c}{ {38.8}}&\multicolumn{2}{c}{ 40.0} &\multicolumn{2}{c}{ {15.1} }&\textbf{40.8} & {$\pm$1.6} \\
			\specialrule{0em}{1.5pt}{1.5pt}
			
			& Hopper          & \cellcolor{white}{16.2} &$\pm$2.5  & \cellcolor{white}{12.7} & $\pm$3.9  &{8.5} &{$\pm$0.0}& \multicolumn{2}{c}{\cellcolor{white}{26.8}}&{31.7}  &  $\pm$0.3 &\multicolumn{2}{c}{\textbf {53.6}}& \multicolumn{2}{c}{\cellcolor{white}{17.8}  }&\multicolumn{2}{c}{21.6}  &\multicolumn{2}{c}{ \cellcolor{white}{11.9} } & \cellcolor{white}{9.2} &$\pm1.4$  \\
			\specialrule{0em}{1.5pt}{1.5pt}
			
			& Walker2d        & \cellcolor{white}{1.2} & $ \pm$0.5 & \cellcolor{white}{2.1} &  $\pm$1.2   &{5.6} &{$\pm$0.0}&\multicolumn{2}{c}{ \cellcolor{white}{8.1} }& \cellcolor{white}{1.7}  & $\pm$0.5  &\multicolumn{2}{c}{\textbf {37.3}}&\multicolumn{2}{c}{ \cellcolor{white}{7.0}  } &\multicolumn{2}{c}{{11.5}} & \multicolumn{2}{c}{\cellcolor{white}{6.4} }  &{10.8} &  $\pm$9.3 \\
			\midrule
			\multirow{3}{*}{\rotatebox[origin=c]{90}{Medium}} & HalfCheetah      & \cellcolor{white}{52.6} & $\pm$0.3 & \cellcolor{white}{48.4} &$\pm$0.3    &{47.7} &{$\pm$0.3}& \multicolumn{2}{c}{\cellcolor{white}{57.9}}&
			\cellcolor{white}{71.1}  & $\pm$2.6  &\multicolumn{2}{c}{42.1}  &\multicolumn{2}{c}{ \cellcolor{white}{54.2} } &\multicolumn{2}{c}{\textbf{77.6}} &\multicolumn{2}{c}{ \cellcolor{white}{45.1}}  &\textbf{73.4} & $\pm$0.5 \\
			\specialrule{0em}{1.5pt}{1.5pt}
			
			& Hopper       & \cellcolor{white}{78.9} &$\pm$6.4  & \cellcolor{white}{56.4} & $\pm$4.9   &{54.3} &{$\pm$4.3}& \multicolumn{2}{c}{\cellcolor{white}{75.3}} &
			\cellcolor{white}{20.7}  &   $\pm$12.9  &\multicolumn{2}{c}{95.4} &\multicolumn{2}{c}{ \cellcolor{white}{94.9} }&\multicolumn{2}{c}{92.8}&\multicolumn{2}{c}{\textbf{100.4}}  & \cellcolor{white}{30.4} &   $\pm$37.4  \\
			\specialrule{0em}{1.5pt}{1.5pt}
			
			& Walker2d        & \cellcolor{white}{82.2} &   $\pm$2.6  & \cellcolor{white}{80.8} &   $\pm$2.9  &{76.1} &{$\pm$5.1} & \multicolumn{2}{c}{\textbf{89.6}}&
			\cellcolor{white}{16.8}  &   $\pm$15.0 &\multicolumn{2}{c}{17.8}& \multicolumn{2}{c}{\cellcolor{white}{77.8}  }&\multicolumn{2}{c}{86.9}&\multicolumn{2}{c}{ \cellcolor{white}{82.0}}  & \cellcolor{white}{55.5} &   $\pm$23.4  \\
			
			\midrule
			\multirow{3}{*}{\rotatebox[origin=c]{90}{\shortstack{Medium\\Replay}}} & HalfCheetah      & \cellcolor{white}{49.5} &   $\pm$0.5  & \cellcolor{white}{44.2} &   $\pm$0.5 &{44.5} &{$\pm$0.5}& \multicolumn{2}{c}{\cellcolor{white}{45.1}}&
			\cellcolor{white}{62.5}  &   $\pm$10.4     &\multicolumn{2}{c}{40.2}& \multicolumn{2}{c}{\cellcolor{white}{55.1} } &\multicolumn{2}{c}{68.9}&\multicolumn{2}{c}{ \cellcolor{white}{44.4}}&
			\textbf{72.8} &   $\pm$0.9  \\
			\specialrule{0em}{1.5pt}{1.5pt}
			
			& Hopper       & \cellcolor{white}{99.2} &   $\pm$1.6  & \cellcolor{white}{56.3} &   $\pm$20.8    &{78.1} &{$\pm$5.3}& \multicolumn{2}{c}{\cellcolor{white}{88.8} } &
			\cellcolor{white}{100.8}  &   $\pm$4.9 &\multicolumn{2}{c}{93.6}  & \multicolumn{2}{c}{\cellcolor{white}{73.1}} &\multicolumn{2}{c}{96.6}&\multicolumn{2}{c}{ \cellcolor{white}{31.3}}  & \textbf{104.6} &   $\pm$1.5 \\
			\specialrule{0em}{1.5pt}{1.5pt}
			
			& Walker2d       & \cellcolor{white}{80.7} &   $\pm$10.7 & \cellcolor{white}{75.7} &   $\pm$7.6   &{68.6} &{$\pm$9.9}& \multicolumn{2}{c}{\cellcolor{white}{77.7}}&
			\cellcolor{white}{80.0}  &   $\pm$8.9   &\multicolumn{2}{c}{49.8} &\multicolumn{2}{c}{ \cellcolor{white}{56.0}} &\multicolumn{2}{c}{85.0}&\multicolumn{2}{c}{ \cellcolor{white}{29.4}} &\textbf{91.1} &   $\pm$2.0\\
			
			\midrule
			\multirow{3}{*}{\rotatebox[origin=c]{90}{\shortstack{Medium\\Expert}}} & HalfCheetah      & \cellcolor{white}{64.2} &   $\pm$11.5   & \cellcolor{white}{86.0} &   $\pm$6.7 &{81.2} &{$\pm$6.0} &\multicolumn{2}{c}{ \cellcolor{white}{92.3}}&
			\cellcolor{white}{80.8}  &   $\pm$11.4   &\multicolumn{2}{c}{53.3} &\multicolumn{2}{c}{ \cellcolor{white}{90.0}}  &\multicolumn{2}{c}{93.7}&\multicolumn{2}{c}{\textbf{105.0}}   & \cellcolor{white}{\textbf{101.5}} &   $\pm$3.1\\
			\specialrule{0em}{1.5pt}{1.5pt}
			
			& Hopper          & \cellcolor{white}{68.2} &   $\pm$25.1 & \cellcolor{white}{100.0} &   $\pm$9.8    &{5.1} &{$\pm$1.6}&\multicolumn{2}{c}{ \cellcolor{white}{110.8} }&
			\cellcolor{white}{21.1}  &   $\pm$20.0   &\multicolumn{2}{c}{108.7} & \multicolumn{2}{c}{\cellcolor{white}{111.1}  }  &\multicolumn{2}{c}{83.3}&\multicolumn{2}{c}{\textbf{112.7} }&\textbf{111.0} &   $\pm0.6$ \\
			\specialrule{0em}{1.5pt}{1.5pt}
			
			& Walker2d      & \cellcolor{white}{109.6} &   $\pm$0.3 &\textbf{110.3} &   $\pm$0.5  &{107.8} &{$\pm$4.0} & \multicolumn{2}{c}{\cellcolor{white}{110.1}}&
			\cellcolor{white}{102.1}  &   $\pm$8.5  &\multicolumn{2}{c}{95.6} & \multicolumn{2}{c}{\cellcolor{white}{96.1}  }  &\multicolumn{2}{c}{68.3}&\multicolumn{2}{c}{ \cellcolor{white}{108.4}}  &\textbf{108.8} &   $\pm$3.2\\
			
			\bottomrule
		\end{tabular}
	\caption{Average normalized score and the standard deviation with the `v2' dataset of D4RL. The highest-performing and competitive scores of our method are highlighted. 
		We run CQL, TD3+BC, IQL, MOPO, and ORPO over 5 different seeds and take the average scores.
		The scores of PBRL, MOReL, COMBO, and CABI are taken from their papers.
	} \label{d4rl-part}
\end{table*}

\subsection{Theoretical Analysis}
\label{sec:theoretical}
Denote the optimal policy in the P-MDP as $\hat \pi^p$.
MOPO has demonstrated that the expected return of $\hat \pi^p$ in the real MDP, denoted as $\eta_{M}(\hat \pi^p)$, has a lower bound.
However, how to train optimal policies in the P-MDP has not been thoroughly investigated.
To bridge this gap, we analyze the optimality of ORPO under the linear-MDPs assumption, which is widely adopted by previous theoretical works~\cite{melo2007q, lsvi-2020,pevi-2021}.

We initially learn a dynamics model and subsequently employ this model to conduct online RL for generating optimistic rollouts.
Based on this point, ORPO aligns closely with the theoretical investigations in online RL, which explore the environment through Upper Confidence Bound (UCB)~\cite{audibert2009exploration}.
From the theoretical perspective, appropriate uncertainty quantification is essential to the provable efficiency in our framework.
We utilize the standard deviation of the dynamics model ensembles for uncertainty quantification, i.e., $u(s,a):=\text{\rm Std}\bigl( \{\hatT_i(s, a)\}_{i=1}^{N}\bigr)$.
Then we can make the following proposition:
\begin{proposition}
	\label{pro:uncertainty}
	Under the assumption of linear MDPs, the uncertainty of dynamics models can form a UCB bonus.
\end{proposition}

We train an optimistic rollout policy for generating model rollouts in the O-MDP.
Since the P-MDP and O-MDP share the same transition distribution, from the view of P-MDP, the reward bonus for training optimistic rollout policy is $(\lambda^p+\lambda^o)  u(s,a)$, which can be a UCB bonus for an appropriately selected tuning $\lambda^o$ and $\lambda^p$.
Then we use the samples (model rollouts) to optimize the output policy $\pi^p$ in the P-MDP.

\begin{theorem}
	\label{the:ORPO}
	Under linear model MDPs and the same assumptions as MOPO, ORPO can find $ \epsilon$-optimal policy $\Ppi$ in the P-MDP, which satisfies
	\begin{equation}
		\eta_{\tM}(\pi^p) \ge \sup_{\pi}\{\eta_{\tM}(\pi) - 2\lambda\epsilon_u(\pi) - \epsilon \}.
	\end{equation}
\end{theorem}

Theorem \ref{the:ORPO} shows that the performance of ORPO in the real MDP can be guaranteed.
Note that we omit the sample complexity because within our framework, samples to optimize the policy can be generated by the learned dynamics model instead of the real environment, which is much cheaper and easier.
We refer to Appendix \ref{appendix-Theoretical} for details.

\begin{table*}[h]
	\centering
		\resizebox{0.86\linewidth}{!}{
		\begin{tabular}{l c c c c c}
			\toprule 
			Environments& CQL & TD3+BC &MOPO & COMBO&ORPO \\ \midrule
			Halfcheetah-jump & 1287.8$\pm$40.4& -4733.3$\pm$746.7& 4411.8$\pm$642.9 & 4595.2$\pm$405.6 &\textbf {5218.0}$\pm$128.5\\
			Halfcheetah-jump-hard & -2989.8$\pm$2.0&-2484.4$\pm$383.3 & -1881.8$\pm$1342.2 & 2782.8$\pm$206.7 &\textbf {4867.9}$\pm$381.6 \\
			\bottomrule
		\end{tabular}
			}
	\caption{
		Average returns over 5 random seeds on tasks that require OOD policy. }
	\label{fig:generalize}
	\normalsize
\end{table*}


\section{Experiments}
In our experiment, we aim to investigate three primary research questions
(RQs): 

\textbf{RQ1 (Performance): }
How does ORPO perform on standard offline RL benchmarks and tasks requiring generalization compared to state-of-the-art baselines?

\textbf{RQ2 (Effectiveness of optimistic rollout policy): }
How does the proposed optimistic rollout policy compare to various other rollout policies?

\textbf{RQ3 (Ablation study):}
How does each design in ORPO affect performance?

To answer the above questions, we conducted our experiments on the D4RL benchmark suite~\cite{fu2020d4rl} as well as two datasets that require generalization to related but previously unseen tasks using the MuJoCo simulator~\cite{todorov2012mujoco}. For the practical implementation of the ORPO algorithm, we utilized the SAC~\cite{SAC} to train the optimistic rollout policy, and for pessimistic offline policy optimization, we used TD3+BC~\cite{TD3BC}. Most of the hyper-parameters were inherited from the optimized MOPO~\cite{lu2022revisiting}. 

\subsection{Performance (RQ1)}
\label{sec:performance}

To answer RQ1, we compared ORPO with several state-of-the-art algorithms, including:
1) CQL~\cite{CQL}: A conservative Q-learning algorithm that minimizes Q-values of OOD actions.
2) TD3+BC~\cite{TD3BC}: A model-free algorithm that incorporates an adaptive behavior cloning (BC) constraint to regularize the policy.
3)  IQL~\cite{IQL}, an implicit conservative Q-learning algorithm to avoid using Q-values of OOD actions.
4) PBRL~\cite{PBRL}: An uncertainty-based algorithm that uses OOD sampling.
5) MOPO~\cite{MOPO}: A model-based algorithm that penalizes rewards based on uncertainty.
6) COMBO~\cite{COMBO}: A model-based variant of CQL.
7) RAMBO~\cite{RAMBO}: A model-based algorithm using robust adversarial RL.
8) CABI~\cite{CABI}: An algorithm that utilizes forward and backward CVAE rollout policies to generate trustworthy rollouts. 

\subsubsection{Results on D4RL benchmarks:}
We summarized the average normalized scores in Table~\ref{d4rl-part}, which includes three environments (HalfCheetah, Hopper, and Walker2d), each with four datasets.
Our ORPO achieved competitive or better results compared to state-of-the-art methods in 8 out of 12 datasets. Overall, ORPO demonstrated significant advantages, particularly when the offline datasets were more diverse, such as in the ``random'' and ``medium-replay'' types. This can be attributed to the improved generalization abilities of the dynamics models trained on such datasets.

We observed that implementing ORPO based on MOPO resulted in a significant performance boost in 11 out of the 12 datasets, increasing the total average normalized score from 610.0 to 809.9, with an improvement of more than 30\%.
The only exception is the ``hopper-random-v2'' datasets. 
This may be because the ``halfcheetah'' and ``walker2d'' tasks are more resilient to OOD actions, while the hopper tasks are more prone to terminating the episode when encountering OOD actions,~\footnote{https://www.gymlibrary.dev/environments/mujoco/}
 making most OOD model rollouts useless.


\subsubsection{Results on tasks requiring generalization:}
To further demonstrate the generalization ability of the output policy, we evaluated on ``Halfcheetah-jump'' dataset proposed by Yu et al.~\cite{MOPO}. This dataset was collected by storing the entire training replay buffer from training SAC for 1 million steps in the HalfCheetah task. The state-action pairs in the dataset were then assigned new rewards that incentivized the halfcheetah to jump.
Based on the ``Halfcheetah-jump'' dataset, we constructed a more challenging dataset, ``Halfcheetah-jump-hard''. This dataset consists of trajectories sampled by a random policy, and the assigned new rewards are further penalized if the halfcheetah is unhealthy. 


We observe that model-based methods show great advantages over model-free.
Notably, ORPO outperforms all the baseline methods by a large margin, highlighting its effectiveness in terms of generalization ability.
In the ``Halfcheetah-jump-hard'' dataset, due to the additional unhealthy penalization on rewards, the policy trained by MOPO is too conservative to run.
ORPO is the only method that achieves satisfactory performance, which suggests that our method can not only generalize to OOD regions but also preclude some of them with low values.

\subsection{Effectiveness of optimistic rollout policy}
\label{sec:rollout}
To answer RQ2, we compare rollout policies in our framework with various rollout policies including:
1) Random rollout policy, which generates actions from the uniform distribution in the action space.
2) Conditional variational autoencoder (CVAE) rollout policy~\cite{CABI}, which offers diverse actions while staying within the span of the dataset.
3) Trained optimistic rollout policy, which uses well-trained fixed optimistic rollout policy to generate optimistic rollouts for pessimistic policy optimization.


\begin{figure}[h]
	\centering
	\includegraphics[width=0.45\textwidth]{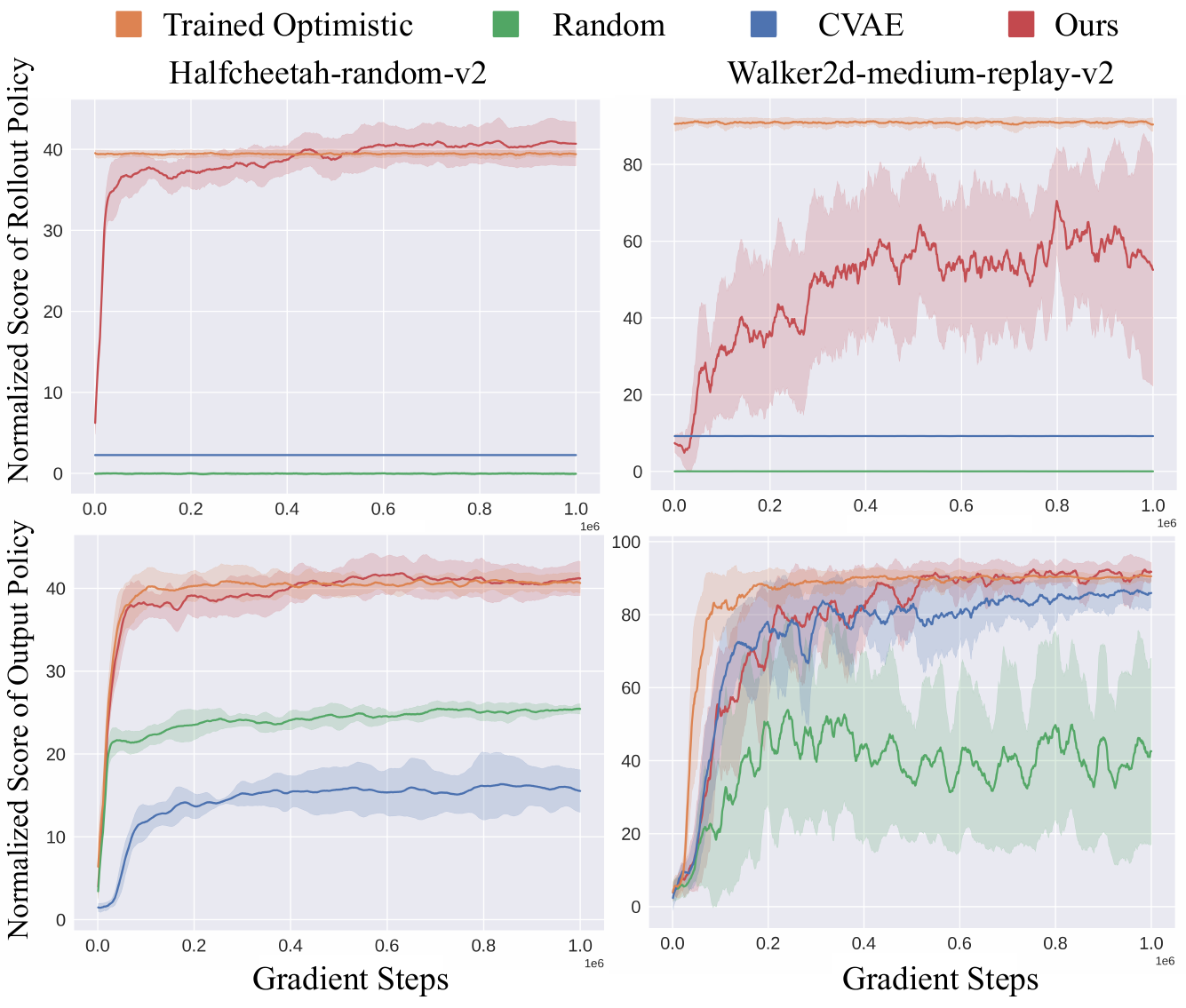}
	\caption{\small Learning curves of rollout policies and corresponding output policies in two datasets over 5 different seeds. }
	\label{fig-Rollout}
\end{figure}

In Figure~\ref{fig-Rollout}, we report the normalized average scores of the rollout policies and output policies in two datasets.
Except for ours, all baseline rollout policies have fixed parameters. So the scores of them are constant.
The score of CVAE rollout policies is only slightly higher than that of the random policies.
This is because CVAE policies are trained to generate rollouts within the support of the offline dataset, while these two datasets conclude many low-value transitions.
In contrast, our optimistic rollout policy can achieve the highest scores due to more valuable model rollouts.


\begin{figure*}[h]
	\centering
	\includegraphics[width=0.95\textwidth]{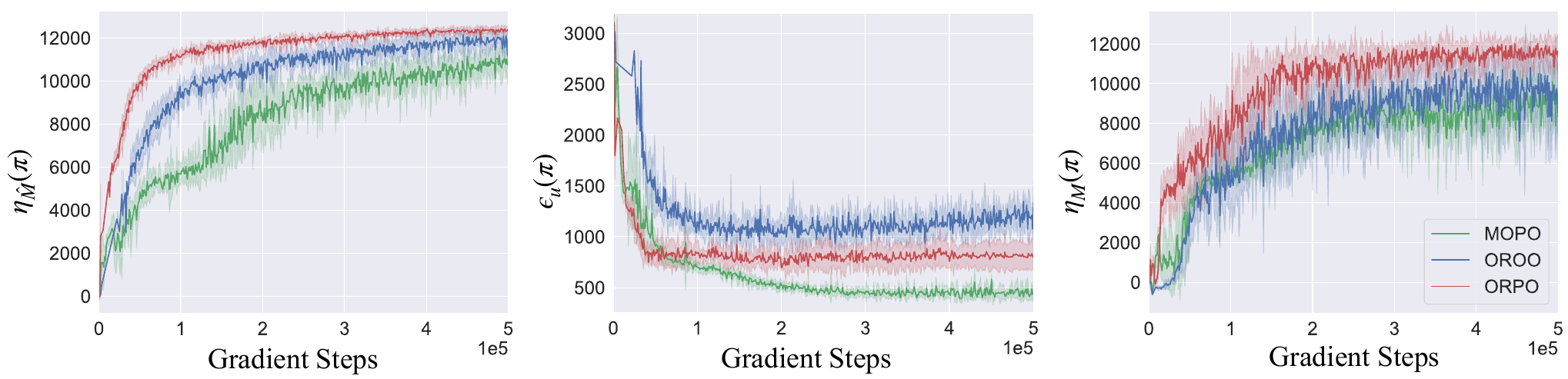}
	\caption{Learning curves of OROO, MOPO, and ORPO over 5 different seeds on ``Halfcheetah-medium-expert-v2''. We report the expected discounted returns in the model MDP $\eta_{\hatM}( \pi)$ and the real MDP $\eta_{M}( \pi)$ as well as the average model error $\epsilon_u( \pi)$. }
	\label{fig-O-MDP-exp}
\end{figure*}

Considering our goal is to achieve high scores for the output policies, the CVAE rollout policy is more effective than the random rollout policy in ``Walker2d-medium-replay-v2'' dataset and vice verse in ``Halfcheetah-random-v2''.
This is because the CVAE rollout policies trained by ``medium-replay'' datasets can generate more valuable rollouts with high-value state-action pairs, but not in ``random'' datasets.
The output policies of ORPO can obtain the highest scores in both datasets.
Though our rollout policy on the ``Walker2d-medium-replay-v2'' dataset can not achieve satisfactory performance, it is still beneficial for optimizing the output policy.
While using fixed well-trained rollouts policy can match our performance, we notice it can be affected by the selected checkpoint.
Therefore, we train optimistic rollout policy and pessimistic output policy iteratively and alternately.




\subsection{Ablation study}
\label{exp:ablation}
\subsubsection{Effect of P-MDP and O-MDP:}We conducted an analysis comparing the expected discounted returns ($\eta_{\tM}(\pi)$ and $\eta_{\hatM}(\pi)$) and the average model uncertainty ($\epsilon_u(\pi)$) of our output policy with the O-MDP and P-MDP baselines, i.e., OROO and MOPO methods.
As shown in Figure~\ref{fig-O-MDP-exp}, we observed that with more OOD sampling, OROO achieves a higher $\eta_{\hatM}(\pi)$ compared to MOPO. 
However, when evaluating OROO in the real environment, we observed a performance significant deterioration due to the noticeable increase in $\epsilon_u(\pi)$, indicating larger model errors. As a result, MOPO remains comparable to OROO in terms of $\eta_{\tM}(\pi)$ on this dataset.

ORPO effectively prevents the agent from accessing risky or potentially dangerous areas. This explains why ORPO achieves higher $\eta_{\hatM}(\pi)$ and lower $\epsilon_u(\pi)$ than OROO. 
Consequently, our method achieves better performance in the real environment ($\eta_M(\pi)$) compared to the baselines that use either P-MDP or O-MDP. Thus, we conclude that ORPO achieves a better trade-off between the generalization ability and estimation errors of the learned dynamics model.

\subsubsection{Sensitivity of the hyper-parameter $\lambda^o$:}
We also conduct experiments to evaluate the sensitivity of ORPO to the hyper-parameter $\lambda^o$, which is used to construct the O-MDPs. 
As shown in Table~\ref{table:lambda}, our results indicate that ORPO achieves satisfactory performance across a wide range of $\lambda^o$ values spanning three orders of magnitude. 
The values of $\lambda^p$ for the two configurations were 4.56 and 2.48, respectively.
Since we optimize the output policies in the P-MDP,  the rollout policy is optimistic as long as $\lambda^o > -\lambda^p$. 
Results show that incorporating optimism can  bring significant performance gains.
Due to the robustness of ORPO to the choice of $\lambda^o$, there is no need to finely tune this parameter for each environment-dataset configuration, and we set $\lambda^o=0.015$ by default for 9 out of the 12 datasets used in the D4RL benchmarks.


\begin{table}[h]
	\centering
		\resizebox{0.48\textwidth}{!}{
			\begin{tabular}{l  c c  c c c c c c}
				\toprule
				&-10&-1 &-0.1&-0.01&0.01&0.1&1&10   \\ 
				\midrule
				H-R &27.7& 39.1 & 40.6 &38.5&40.2& 41.2&41.6&26.0  \\
				W-M-R &43.7 & 85.5  & 91.1&91.3&91.8& 88.6& 8.6&0.1 \\ 
				\bottomrule
			\end{tabular}
		}
		\caption{
			Ablation of the different optimism hyper-parameter $\lambda^o$ for ORPO.
			``H-R'' represents HalfCheetah-Random-v2 and ``W-M-R'' represents Walker2d-Medium-Replay-v2.}
		\label{table:lambda}
	\end{table}


\subsubsection{Other ablation studies:}
We briefly report the results compared to other baselines. 
1) We compare to MOPO (TD3+BC) which replaces SAC in MOPO with TD3+BC for policy optimization.
The results migrate the effect of different RL algorithms on performance gain over MOPO and suggest the effectiveness of optimistic rollouts.
2) We compare ORPO to ORPO (SAC), which use SAC to optimize both rollout policies and output policies, and demonstrate the effectiveness of utilizing offline RL algorithms for pessimistic policy optimization.
3) We compare to ORPO without pessimism which replaces the P-MDP used in ORPO with the model MDP, and demonstrate the necessity of pessimism in ORPO.
Complete results can be found in Appendix~\ref{appendix:baselines}.


\section{Conclusion and Limitations}
	\label{sec:conclusion}
In this paper, we started with the observation that incorporating optimism when generating model rollouts can yield benefits for model-based offline RL. Building upon this insight, we have introduced ORPO, a novel framework that leverages optimistic model rollouts for pessimistic policy optimization. The theoretical analysis of ORPO demonstrates its efficiency in addressing the challenges of offline RL. Through extensive empirical evaluations, we have demonstrated that ORPO significantly enhances the performance of the P-MDP baseline and surpasses state-of-the-art methods on both the D4RL benchmark and tasks demanding generalization. 

	
	Our work has limitations.
	One limitation is the additional time overhead required for training optimistic rollout policies. Additionally, in tasks where OOD actions are strictly prohibited, ORPO may have negative effects compared to existing P-MDP baselines. Therefore, one future direction is to explore an adaptive degree of optimism when evaluation\cut{ to address this limitation}.


\section{Acknowledgements}
This work was supported by the National Key R\&D Program of China (No. 2021ZD0112904).
	
\bibliography{ORPO}

\newpage
\onecolumn
%
\appendix
\onecolumn
\newcommand{\RR}{\mathbb{R}}

\section{Theoretical Supports for ORPO}
\label{appendix-Theoretical}

\subsection{Proof of Proposition~\ref{pro:uncertainty}}
In this section, we study the model-based uncertainty quantification under the linear MDP assumptions \cite{bradtke1996linear, melo2007q}. 

\begin{assumption}[Linear model MDP] \label{assumption:linear}
	The model MDP$(\cS, \cA, \hat T, r, H, \gamma)$ is a \emph{linear MDP} with a feature map $\phi: \cS \times \cA \rightarrow \mathbb{R}^d$, where $H$ is the length of each episode and we set $\gamma = 1$ for simplicity.
	 For any $h\in [H]$, there exist $d$ \emph{unknown} (signed) measures $w$ over $\cS$ and an \emph{unknown} vector $\mu \in \mathbb{R}^d$, such that 
	for any $(s, a) \in \cS \times \cA$, we have 
	\begin{align}\label{eq:linear_transition}  
		\hat T(\hat s' |s, a) = \phi(s, a)w,   \qquad 
		\hat T(r |s, a) = \phi(s, a) \mu.  
	\end{align}
	Following previous works, we assume $||{\phi(s, a)}\|| \le 1$ for all $(s,a ) \in \cS \times \cA$, and $\max\{||{w(\cS)}||, ||\mu||\} \le \sqrt{d}$ for all $h \in [H]$.
\end{assumption}
In model-based offline RL, dynamics models are typically learned in the supervised manner.
The parameter $w$ can be solved in the closed form by following the Least-Squares Value Iteration (LSVI) algorithm~\cite{bradtke1996linear}, which minimizes the following loss function,
\begin{equation}
	\label{eq::appendix_OOD_LSVI}
	\widehat{w} = \min_{w \in \RR^{d}} \sum^m_{i = 1}\|\phi(s^i,a^i){w} -  s'^i\|^2 +  \beta \cdot \|w\|_2^2,
\end{equation}
where $s'^i$, as the target of LSVI, is the next state for $(s^i,a^i)$ in the offline dataset. 
$m$ represents the effective number of samples the agent has observed so far along the $\phi$ direction,
The explicit solution to (\ref{eq::appendix_OOD_LSVI}) takes the form of
\begin{equation}
	\label{eq::pf ilde_w}
	\widehat{w} = \Lambda^{-1}\sum^m_{i = 1}\phi(s^i,a^i)s'^i,\quad \Lambda  = \sum^m_{i=1}\phi(s^i,a^i)\phi(s^i,a^i)^\top + \beta \cdot \mathrm{\mathbf{I}},
\end{equation}
where $\Lambda $ accumulates the state-action features from the offline dataset.  
From the theoretical perspective, appropriate uncertainty quantification is essential to the provable efficiency in our framework.
Therefore, we make the following proposition:
\begin{proposition}
\label{uncertainty}
With linear function assumption, $\phi(s,a)^\top\Lambda^{-1}\phi(s,a)$ yields model-based uncertainty quantification.
\end{proposition}

\begin{proof}
	Training dynamics models is a regression task.
	We consider a Bayesian linear regression perspective of LSVI in Eq.~\ref{eq::appendix_OOD_LSVI}.
	We further define the noise $\sigma$ in this least-square problem as follows,
	
	\begin{equation}\label{eq:yt}
		\sigma=s'^i - \phi(s^i,a^i)^\top{w},
	\end{equation}
	where $w$ is the underlying true parameter. In the offline dataset with $\D$, we denote by $\widehat{w}$ the Bayesian posterior of $w$ given the dataset $\D$. In addition, we assume that we are given a multivariate Gaussian prior of the parameter $w \sim \mathcal N(\mathbf{0}, \mathrm{\mathbf{I}}/\lambda)$ as a non-informative prior and noise $\sigma\sim \mathcal N(0, 1)$.
	Then we obtain that 
	
	\begin{equation}\label{eq::pf_density_y}
		s'^i \:|\: \bigl((s^i, a^i), \widehat{w}\bigr) \sim \mathcal{N}\bigl(\phi(s^i, a^i)^\top \widehat{w}, 1\bigr).	
	\end{equation}
	Our objective is to compute the posterior density $\widehat{w} = w \:|\: \D$. It holds from Bayes rule that 
	\begin{equation}\label{eq::pf_bayes_rule}
		\log p(\widehat{w} \:|\: \D) = \log p(\widehat{w}) + \log p(\D \:|\: \widehat{w}) + {\rm Const},
	\end{equation}
	where $p(\cdot)$ denotes the probability density function of the respective distributions. Plugging  Eq. (\ref{eq::pf_density_y}) and the probability density function of Gaussian distribution into Eq.(\ref{eq::pf_bayes_rule}) yields
	\begin{equation}\label{eq::pf_density_posterior}
		\begin{aligned}
			\log p(\widehat{w} \:|\: \D) &= -\|\widehat{w}\|^2/2 -\sum^m_{i=1} \| \phi(s^i, a^i)^\top \widehat{w} - s'^i\|^2/2 + {\rm Const.}\notag\\
			&=-(\widehat{w} - \mathbf{\mu})^\top \Lambda^{-1}(\widehat{w} - \mathbf{\mu})/2 + {\rm Const.},
		\end{aligned}
	\end{equation}
	where we define 
	\begin{equation}
		\mathbf{\mu}= \Lambda^{-1}  \sum^m_{i = 1}\phi(s^i, a^i) s'^i, \qquad \Lambda =\sum_{i=1}^{m}\phi(s^i,a^i)\phi(s^i,a^i)^\top+\beta \cdot \mathrm{\mathbf{I}}.
	\end{equation}
	Then we obtain that $\widehat{w} = w\:|\: \D \sim \mathcal N(\mathbf{\mu}, \Lambda^{-1})$. It then holds for all $(s, a)\in\mathcal{S}\times\mathcal{A}$ that
	\begin{equation}
		\label{222}
		\text{\rm Var}\bigl(\phi(s, a)^\top \widehat{w} \bigr)  = \phi(s, a)^\top \Lambda^{-1} \phi(s, a). 
	\end{equation}

	
We can utilize the following uncertainty quantification $u(s,a)$ for the model error estimator, which is captured by the approximate predictive standard deviation (Std) with respect to the ensembles of the learned dynamics model. 
	\begin{equation}
		\label{333}
		u(s, a):=\text{\rm Std}\bigl( \hatT(s, a)\bigr)=\text{\rm Std}\bigl(\phi(s, a)^\top \widehat{w} \bigr)  =\big[ \phi(s, a)^\top \Lambda^{-1}  \phi(s, a )\big]^{\nicefrac{1}{2}}.
	\end{equation}
	
\end{proof}

In Proposition~\ref{uncertainty}, we show that the uncertainty quantified by the standard deviation has the form $\big[ \phi(s, a)^\top \Lambda^{-1}  \phi(s, a)\big]^{\nicefrac{1}{2}}$, which can be used to construct the Upper Confidence Bound (UCB).
Exploration with UCB as a bonus is provably efficient in online RL~\cite{bandit-2011,lsvi-2020}.
While in the offline RL setting, such uncertainty quantification can also be used to construct the Lower Confidence Bound (LCB) penalty:
\begin{equation}
	\label{ucb+lcb}
	\begin{array}{cc}
		& 
		 \Gamma^{\rm lcb}(s ,a )=\lambda^p \cdot\big[\phi(s ,a )^\top\Lambda ^{-1}\phi(s ,a )\big]^{\nicefrac{1}{2}},
	\end{array}
\end{equation}
which measures the confidence interval of the dynamics models with the given training data.
Besides, using $\Gamma^{\rm lcb}(s ,a )$ for pessimistic LSVI is known to be information-theoretically optimal in offline RL~\cite{pevi-2021}.


We remark that there are two options to apply $\Gamma^{\rm ucb}(s ,a )$ or $\Gamma^{\rm lcb}(s ,a )$: on the next-Q value or the immediate reward.
Considering the update of Q functions, the two options have the same influence with tuning weights.
To maintain consistency with previous model-based offline works, we use the latter option to construct the P-MDP and O-MDP.

\subsection{Linear P-MDP}
In this section, we prove that under the assumption of linear model MDP, the constructed P-MDP is also a linear MDP.

Firstly, we define the Bellman operator based on the value function as 
\begin{equation}
	\mathcal{T}V_{\tM}(s,a):=\Esub{s'\sim\tT(s,a)}[r(s,a)+V^\pi_{\tM}(s')],
\end{equation}
and define $\widehat \cT$ as the empirical Bellman operator that estimates $\cT$ based on the dynamics model $\hat T$.
\begin{equation}
	\mathcal{\hat T}V_{\tM}(s,a):=\Esub{s'\sim\hat T(s,a)}[r(s,a)+V_{\tM}(s')].
\end{equation}
According to Equation~\ref{equation:G}, the model error defined in MOPO~\cite{MOPO} can also be expressed as:
\begin{equation}
	|G_{\hatM}(s,a)| =  | \mathcal{T}V_{\tM}(s,a) - \mathcal{\widehat T}V_{\tM}(s,a) |.
\end{equation}
We introduce the $\xi$-Uncertainty Quantifier, which plays an important part in the theoretical analysis of both online and offline RL.
\begin{definition}[$\xi$-Uncertainty Quantifier]\label{def1}
	The set of $\Gamma^{\rm lcb}(s , a )$ $ (\Gamma^{\rm lcb}:\cS\times\cA\to \mathbb{R}^1)$ forms a $\xi$-Uncertainty Quantifier if it holds with probability at least $1 - \xi$ that
\begin{equation}
		|\widehat \cT V_{\tM}(s, a) - \cT V_{\tM}(s , a )| \leq \Gamma^{\rm lcb}(s , a ),
\end{equation}
	for all $(s, a)\in\cS\times\cA$.
\end{definition}

Recent theoretical works shows that $\Gamma^{\rm lcb}(s , a )$ in Equation~\ref{ucb+lcb} can be a $\xi$-Uncertainty Quantifier for appropriately selected $\lambda^p$, which we show in the following lemma.
\begin{lemma}[ \citep{pevi-2021}]
	\label{lem::xi}
 With the assumption of the complaint data-collecting process, if we set $\beta=1$, it holds for $\lambda^p = c\cdot dH\sqrt{\zeta}$ that
	\begin{equation}
 \Gamma^{\rm lcb}(s ,a )=\lambda^p\big[\phi(s ,a )^\top\Lambda ^{-1}\phi(s ,a )\big]^{\nicefrac{1}{2}}
	\end{equation}
	forms a valid $\xi$-uncertainty quantifier, where $\zeta= \log(2dHK/\xi)$, $\xi \in (0,1)$ is the confidence parameter, $d$ is the ambient dimension of feature space, $K$ is the number of episodes in the offline data, and $H$ is the length of each episode.
\end{lemma}

According to Lemma~\ref{lem::xi} and the condition of admissible model uncertainty in Equation~\ref{equ:condition}, we can construct P-MDP based on $\Gamma^{\rm lcb}(s , a )$ by setting $\xi$ to approach 1.
\begin{proposition}
	A linear MDP penalized by $\Gamma^{\rm lcb}(s ,a )=\lambda_p u(s,a)$ forms a P-MDP proposed in MOPO~\cite{MOPO}, which is also a linear MDP.
\end{proposition}

\begin{proof}
Note that MOPO provides a lower bound on the expected discounted return in the real MDP by the expected discounted return in the model MDP.
\begin{equation}
	\label{supple:equ:lower-bound}
	\eta_{\tM}(\pi) = \Eisub{(s,a) \sim \rho^\pi_{\hatT}}\left[ r(s,a) - \gamma G^\pi_{\hatM}(s,a)\right]  \ge \Eisub{(s,a) \sim \rho^\pi_{\hatT}}\left[ r(s,a) - \gamma |G^\pi_{\hatM}(s,a)|\right].
\end{equation}
From Lemma~\ref{lem::xi}, we know:
\begin{equation}
		 |G_{\hatM}(s,a)| = |\widehat \cT V(s, a) - \cT V(s , a )| \leq \Gamma^{\rm lcb}(s , a) = \lambda_p u(s , a) ,
		 \label{temp2}
\end{equation}
Plugging Eq. (\ref{temp2}) into Eq. (\ref{supple:equ:lower-bound}) yields
\begin{equation}
	\label{supple:equ:P-MDP}
	\eta_{\tM}(\pi) \ge \Eisub{(s,a) \sim \rho^\pi_{\hatT}}\left[ r(s,a) - \lambda^p u(s,a) \right] =\eta_{\hatM}(\pi)-\lambda^p \epsilon_u(\pi) = \eta_{\pM}(\pi).
\end{equation}
which indicates that $\Gamma^{\rm lcb}(s , a)$ can be an admissible model uncertainty to construct the P-MDP.

Next we prove the constructed P-MDP is also a linear MDP.
Since $\Lambda =\sum_{i=1}^{m}\phi(s ^i,a ^i)\phi(s ^i,a ^i)^\top+\beta \cdot \mathrm{\mathbf{I}}$ is a symmetric matrix, $\Lambda ^{-1}$ is also a symmetric matrix.
	So we can diagonalize it as
	\begin{equation}
		\Lambda^{-1}  = P  \tilde \Lambda  P^\top,
	\end{equation}
where $P $ is a orthogonal matrix and $\tilde \Lambda $ is a diagonal matrix.
Then we have:
\begin{equation}
	\phi(s ,a )^\top\Lambda ^{-1}\phi(s ,a ) = \phi(s ,a )^\top P  \tilde \Lambda  P^\top  \phi(s ,a ) = (P^\top  \phi(s ,a ) )^\top  \tilde \Lambda  (P^\top  \phi(s ,a ) )
\end{equation}

Considering the definition of $\Gamma^{\rm lcb}(s ,a )$ in Equation~\ref{ucb+lcb}, $\Gamma^{\rm lcb}(s ,a )$ is linear with $P^\top  \phi(s ,a )$, and is linear with $\phi(s, a)$.
Therefore, $r^p = r - \lambda^p u(s,a)= \phi(s,a)^\top\widehat{w} - \lambda^p \cdot\big[\phi(s ,a )^\top\Lambda ^{-1}\phi(s ,a )\big]^{\nicefrac{1}{2}}$ is also linear with $\phi(s ,a )$.
The only difference between the model MDP and the P-MDP is the reward function, so the transition kernel is still linear with $\phi$. 
Therefore, the constructed P-MDP is also a linear MDP.
\end{proof}

\subsection{Proof of Theorem~\ref{the:ORPO}}

In our ORPO framework, we optimize the output policies with model-free manner in the P-MDP with reward bonus according to the O-MDP.
Since our constructed P-MDP is also linear, we can provide the LSVI version of ORPO.

\begin{algorithm}[h]
	\caption{Least-Squares Value Iteration with ORPO (LSVI-ORPO)}\label{algo:LSVI-ORPO}
	\begin{algorithmic}[1]
		\FOR{episode $k = 1, \ldots, K$}
		\STATE Receive the initial state $s^k_1$.
		\FOR{step $h = H, \ldots, 1$}
		\STATE $m \leftarrow k-1$
		\STATE $\Lambda^p_h \leftarrow \sum_{i =1}^{m}  \phi(s^{i}_h,  a^{i}_h)\phi(s^{i}_h, a^{i}_h)^\top + \beta \cdot \mathrm{\mathbf{I}}$. \label{line:Lambda}
		\STATE $w^p_h \leftarrow {\Lambda^p_h}^{-1} \sum_{i=1}^{m} \phi(s^{i}_h,  a^{i}_h) [r^{p}_h(s^{i}_h, a^{i}_h) + \max_a Q_{h+1}(s^{i}_{h+1}, a)]$. \label{line:w}
		\STATE $Q_h(\cdot, \cdot) \leftarrow \min\{{w^p_h}^\top\phi(\cdot, \cdot) + (\lambda^o+\lambda^p)  [\phi(\cdot, \cdot)^\top {\Lambda^p_h}^{-1} \phi(\cdot, \cdot)]^{1/2}, H\}$. \label{line:ucb}
		\ENDFOR
		\FOR{step $h = 1, \ldots, H$}
		\STATE Take action $a^k_h \gets  \argmax_{a \in \cA } Q_h(s^k_h, a)$, and observe $s^k_{h+1}$. 
		\ENDFOR
		\ENDFOR
	\end{algorithmic}
\end{algorithm}

Note that Algorithm~\ref{algo:LSVI-ORPO} is from the view of the constructed linear P-MDP, and the O-MDP is used to provide reward bonus $(\lambda^o+\lambda^p)  [\phi(\cdot, \cdot)^\top {\Lambda^p_h}^{-1} \phi(\cdot, \cdot)]^{1/2}$ to ``explore'' the P-MDP.
We define the value function in the P-MDP $V_{M^p}^{\pi}\colon \cS \to \mathbb{R} $ as the expected value of cumulative  rewards received under policy $\pi$ when starting from an arbitrary state at the $h$th step. Specifically, we have 
\begin{equation}
	V_{M^p}^{\pi}(s) := \E\left[\sum_{h' = h}^H r_{h'}(s_{h'}, \pi(s_{h'}, h'))  \bigg | s_h = s\right], \qquad \forall s\in \cS, h \in [H].
\end{equation}
Then the total regret of the output policy of ORPO $\pi^p$ in the P-MDP can be defined using the optimal policy in the P-MDP $\pi^\star$:
\begin{equation}
	\text{Regret}(K) = \sum_{k=1}^K \left[V_{M^p}^{\pi^p_\star}(s_1^k) -V_{M^p}^{\pi^p_k} (s_1^k)\right],
\end{equation}
where $K$ is the total number of episodes, $\pi^p_k$ is the policy in  the k-th episode and $\pi^p_\star$ is the optimal policy in the P-MDP.

\begin{lemma}[]
	\label{supple:prop:efficiency}
	Under the linear P-MDP, there 
	exists an absolute constant $c >0$ such that, for any fixed $p \in (0, 1)$, if we set $\beta = 1$ and  $(\lambda^o+\lambda^p) =  c \cdot d H \sqrt{\iota} $  with $\iota := \log (2dT/p)$,   
	then with probability $1-p$, the total regret of LSVI-ORPO in the P-MDP is at most $\cO(\sqrt{d^3 H^3T \iota^2})$, where $\cO(\cdot)$ hides only absolute constants, $d$ is the ambient dimension of feature space, $H$ is the length of each episode, and $T$ is the total number of steps. 
\end{lemma}
\begin{proof}
	See e.g., \cite{lsvi-2020} for a detailed proof.
\end{proof}

Lemma~\ref{supple:prop:efficiency} shows that the total regret of Algorithm~\ref{algo:LSVI-ORPO} can be upper-bounded.
Such regret bound directly translates to a sample complexity guarantee, or a probably approximately correct (PAC) guarantee in terms of the optimal policy.

\begin{lemma}[]
	\label{supple:prop:regret2return}
	With the fixed initial state $s_1$, with at least constant probability, LSVI-ORPO can learn an $\epsilon$-optimal policy $\pi^p$, which satisfies $V_{M^p}^{\pi^p_\star}(s) - V_{M^p}^{\pi^p}(s) \le \epsilon$ using $\ccO(\sqrt{d^3 H^4/\epsilon^2})$ samples.
\end{lemma}
\begin{proof}
	See e.g., \cite{jin2018q} for a detailed proof.
\end{proof}

Now we prove Theorem~\ref{the:ORPO}.

\begin{proof}
As an immediate corollary of Lemma~\ref{supple:prop:regret2return}, we can learn an $\epsilon$-optimal policy $\pi^p$ which satisfies
\begin{equation}
	\label{optimal}
	\eta_{\pM}(\pi^p) \ge \eta_{\pM}(\hat \pi^p) - \epsilon
\end{equation}
using $\ccO(|S|\sqrt{d^3 H^4/\epsilon^2})$ samples, where $\epsilon \in (0,H]$ is a constant, and $|S|$ is the cardinality of the state space. Note that there can be tighter bound for the number of samples. However, the samples of Algorithm~\ref{algo:LSVI-ORPO} is generated by our learned dynamics model instead of sampling from the environment.  So we do not delve into the sample complexity here.

From Equation~\ref{equ:lower-bound}, we have 
\begin{equation}
	\label{two-side}
		|\eta_{\hatM}(\pi) -  \eta_{M}(\pi) | \le \lambda^p \epsilon_u(\pi).
\end{equation} 

Then we have, for any policy $\pi$, with at least constant probability,
	\begin{align}
		\eta_{\tM}(\pi^p) &\ge \eta_{\tilM}(\pi^p) \tag{by~(\ref{equ:lower-bound})} \\
		&\ge  \eta_{\pM}(\hat \pi^p) - \epsilon \tag{by (\ref{optimal})}\\
		&\ge \eta_{\tilM}(\pi) - \epsilon \tag{by definition of $\hat \pi^p$} \\
		&= \eta_{\hatM}(\pi) - \lambda \epsilon_u(\pi) - \epsilon \nonumber \\
		&\ge \eta_{\tM}(\pi) - 2\lambda\epsilon_u(\pi) - \epsilon \tag{by (\ref{two-side})}
  \label{equation:lower-bound}
	\end{align}

ORPO can always find $\epsilon$-optimal policies, but can only output a policy that satisfies the lower bound \ref{equation:lower-bound} with a certain probability. This is because we can only select one policy from the set of policies generated during optimization in Algorithm 2.
In fact, when running $K$ episodes and randomly selecting $\pi = \pi_k$ for $k=1,2,\cdots,K$, $\pi$ can be lower-bounded with a probability of at least $2/3$~\cite{jin2018q}.

\end{proof}

\quad
\section{Implementation Details}
\label{Supp-prac}
We now present a practical instantiation of ORPO (Algorithm~\ref{broad-alg}).
The principal differences are the specialization of uncertainty quantifiers and the algorithms for training rollout and output policy.

\subsection{Uncertainty Heuristics}
In our theoretical analysis, we use the predictive standard deviation of the model  $\widehat{T}(\hat s' |s,a)$ to approximate the model error:
\begin{align}
u(s,a):=&Std[\hat T (s,a) ] \nonumber \\ \approx & \underbrace{{\sigma_A}}_{\text{Aleatoric Uncertainty}} + \underbrace{\sqrt{{\Esub{(s,a)\sim \rho^\pi_{\hatT}}(\hatT (s,a)^\top \hatT (s,a)}) - {E[\hatT (s,a)]^\top E[\hatT (s,a)]}}}_{\text{Epistemic Uncertainty}},
\end{align}
where ${\sigma_A}$ corresponds to the aleatoric uncertainty which stems from the behavior and interactions of the environment, while the second term corresponds to the epistemic uncertainty which means how much the model is uncertain about its predictions~\cite{UWAC,lockwood2022review}.

In practice, following prior model-based works, we use an ensemble of $N$ probabilistic dynamics models  $\widehat T = \{\widehat{T}_i(\hat s', \hat r |s,a)  = \mathcal{N}(\mu_i (s,a) , \Sigma_i (s,a) )\}_{i=1}^N$, and adopt several uncertainty heuristics to estimate $u(s,a)$.

\textbf{Max Aleatoric:}
$max_{1,2,\cdots,N}\Vert \Sigma_i (s,a) \Vert_F$, which corresponds to the maximum aleatoric error that captures both the epistemic and aleatoric uncertainty of the true dynamics.

\textbf{Ensemble Variance (Ensemble Var):}
${E(\mu_i (s,a)^\top \mu_i (s,a)+\Sigma_i (s,a)^\top \Sigma_i (s,a)}) - {E[\mu_i (s,a)]^\top E[\mu_i (s,a)]}$, which is a combination of epistemic and aleatoric model uncertainty.

\textbf{Ensemble Standard Deviation (Ensemble Std):}
The square root of the ensemble variance.

Lu et al.~\cite{lu2022revisiting} empirically shows that there is no uniform uncertainty heuristic that performs best in all environments.
They also reported the referenced heuristics and corresponding $\lambda^p$.
We adopt most of them and list all of the hyper-parameters in Appendix~\ref{supple-hyperparameters}.

\subsection{Practical Algorithmic Outline}
\label{appendix:practical}
In this section, we implement the practical ORPO algorithm based on MOPO~\cite{MOPO}.
Nevertheless, our method is not limited to MOPO and can apply to other model-based offline methods using P-MDPs like MOReL~\cite{MOReL}.
As outlined in Algorithm~\ref{detalied-alg}, the differences from MOPO are \textbf{highlighted}.
We utilize the optimistic rollout policy to roll out (lines 8-14).
Besides, we utilized the SAC algorithm~\cite{SAC} to train the optimistic rollout policy (line 24), and used TD3+BC for pessimistic offline policy optimization (line 25). 

\begin{algorithm}
	\caption{ORPO: Optimistic Rollout and Pessimistic Policy Optimization}\label{detalied-alg}
	\small
	\begin{algorithmic}[1]
		\STATE {\bfseries Input:} Offline dataset $\D$, coefficient for constructing P-MDP and \textbf{O-MDP} $\lambda^p, {\lambda^o}$, rollout horizon $h$, rollout batch size $b$. 
		\STATE Train on batch data $\D$ an ensemble of $N$ probabilistic dynamics $\widehat T = \{\widehat{T}_i(\hat s' |s,a)  = \mathcal{N}(\mu_i (s,a) , \Sigma_i (s,a) )\}_{i=1}^N$, and compute the mean of $\{\hat T_i\}_{i=1}^N$ as the dynamics model $\hat T$.
		\STATE Initialize the output policy $\Ppi$, \textbf{the rollout policy $\Opi$}, empty optimistic  buffer $\oD \leftarrow \varnothing$ and pessimistic replay buffers $\ppD \leftarrow \varnothing, \poD \leftarrow \varnothing$.
		\FOR{epoch $1, 2, \dots$}
		\FOR{$1, 2, \dots, b$ (in parallel)}
		\STATE Sample state $s_1$ from $\D$ for the initialization of the rollout.   
		\FOR{$j = 1, 2, \dots, h$}  
		\STATE \textbf{Sample an action $a^o_j \sim \pi^o(s_j)$.}
		\STATE \textbf{Sample $\hat s_{j+1},  r_j \sim \hat T(s_j,a_j)$.}
        \STATE \textbf{Compute $u(s_j,a_j)$ using a uncertainty heuristic.}
		\STATE \textbf{Compute ${r}^{p}_j = r_j {- \lambda_p u(s_j,a_j)}$. }
		\STATE \textbf{Compute ${r}^{o}_j = r_j {+ \lambda_o u(s_j,a_j) }$.} 
		\STATE \textbf{Add sample $(s_j, a_j, {r}^{p}_j, s_{j+1})$ to $\poD$.}
		\STATE \textbf{Add sample $(s_j, a_j, {r}^{o}_j, s_{j+1})$ to $\oD$.}
		\ENDFOR
		\FOR{$j = 1, 2, \dots, h$}
		\STATE Sample an action $a^p_j \sim \pi^p(s_j)$.
		\STATE  Sample $\hat s_{j+1},  r_j \sim \hat T(s_j,a_j)$.
		\STATE Compute $u(s_j,a_j)$ using a uncertainty heuristic.
		\STATE Compute ${r}^{p}_j = r_j {- \lambda_p u(s_j,a_j)}$. 
		\STATE Add sample $(s_j, a_j, {r}^{p}_j, s_{j+1})$ to $\ppD$.
		\ENDFOR
		\ENDFOR
		\STATE \textbf{Draw samples from $ \D \cup \oD$, and use SAC to update $\pi^o$. }
		\STATE \textbf{Draw samples from $ \D \cup \poD \cup \ppD$, and use TD3-BC to update $\pi^p$. } 
		\ENDFOR
		\STATE {\bfseries Output: }Optimized output policy $\Ppi$.
	\end{algorithmic}
\end{algorithm}

\quad
\section{Experiment Details}
\label{Supp-exp}

\subsection{Toy Example}
\label{exp-toy}
In this section, we give the detailed experimental setup of our toy task.
Our toy task is a 2-dimensional, continuous state space, continuous action space environment as shown in Figure \ref{fig-toy-scene}. The central point of the square region is $(0,0)$, and the state space gives $D \coloneqq [-3,3]\times[-3,3]$. 
The state information is composed of the coordinates of the agent, i.e., $s=(x,y), x,y\in[-3,3]$. 
Each episode, the agent randomly starts at the region between lines $y=-x-0.25$ and $y=-x+0.25$.
The goal is to move upper right to obtain high rewards and takes actions $a\in[-1,1]$. 
 The reward function $r(s,a)$ is depends on the distance to the line $y=-x$, which is defined in (\ref{eq:reward}).
\begin{equation}
	\label{eq:reward}
	r(s,a) = \begin{cases}
		\frac{|x+y|}{\sqrt{2}}, \quad \mathrm{if} \, x\ge-y, \\
		\frac{|x+y|}{\sqrt{2}}, \quad \mathrm{if} \, x<-y.
	\end{cases}
\end{equation}

The agent is not allowed to step out of the square region. The episode length for RiskWorld is set to be 10.
We initialize the agent randomly, take a random action and reset the agent to collect an offline dataset. This prove that the states of all transitions in the offline dataset is within the region between lines $y=-x-0.25$ and $y=-x+0.25$, and the next states near the line $y=-x$. Figure~\ref{fig:toy-distribution} shows the state distribution (blue) of the dataset.

The learned dynamics model $\hat T$ receives the current state and action as input, and outputs a multivariate Gaussian distribution that predicts the next state and reward.
The neural networks and the hyper-parameters are the same to other tasks in this work.
We estimate the uncertainty using the Ensemble Std of dynamics models.
Since the state distribution of the offline dataset is near the line $y=-x$, the uncertainty value decrease as the distance from the straight line increases, as shown in Figure~\ref{fig-toy-uncertainty}.

We set $\lambda_p=100$ for both MOPO and ORPO, and set $\lambda_p=1$ for ORPO.
Considering the simplicity of the toy task, we train both the MOPO and ORPO policies for 10 epochs containing 10000 gradient steps.
We evaluate the MOPO policy and the output policy for 5 episodes and visualize trajectories in Figure~\ref{fig-toy-mopo} and Figure~\ref{fig-toy-orpo} respectively.

We also evaluate both policies in the real environment for total 5000 time steps.
The visited states are shown in Figure~\ref{supple:trajectory_mopo} and Figure~\ref{supple:trajectory_orpo} respectively.
The policy trained with MOPO~\cite{MOPO} utilizing only P-MDP can not  reach regions with high reward but high uncertainty, where the average  cumulative reward received is only $7.3$.
With more OOD sampling and pessimistic policy optimization, ORPO agent can learn to reach states with high rewards and avoid regions with low reward, achieving the average cumulative reward $35.3$.

\begin{figure*}[!ht]
		\begin{center}
		\subfigure[Evaluation transitions of the policy trained in the P-MDP.
		\label{supple:trajectory_mopo}]
		{
			\centering
			\includegraphics[width=0.3\linewidth]{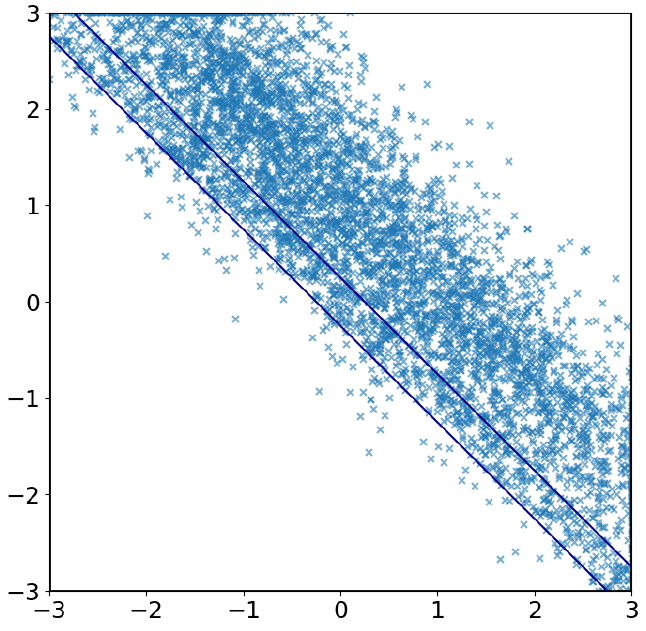}
		}
		\quad
		\subfigure[Evaluation transitions of the policy trained with ORPO.
		\label{supple:trajectory_orpo}]
		{
			\centering
			\includegraphics[width=0.3\linewidth]{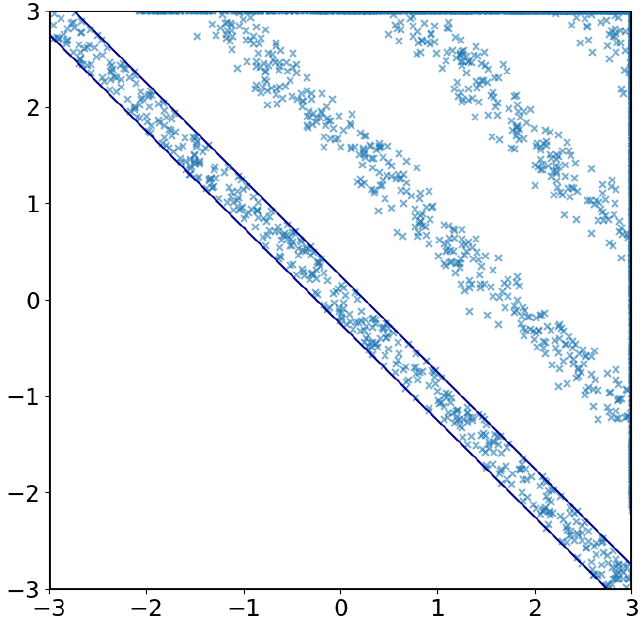}
		}
		\caption{
		Visited states after evaluation for 5000 time steps in the toy task.
		}\label{fig:toy-distribution}
			\end{center}
\end{figure*}

\subsection{D4RL Benchmark}
\label{exp-env}
The D4RL dataset~\cite{fu2020d4rl} is a collection of interactions with continuous action environments simulated on the MuJoCo physics engine~\cite{todorov2012mujoco}. Three tasks, including halfcheetah, hopper, and walker2d, are utilized in this dataset as shown in Figure \ref{fig:mujocoviaual}. The dataset is divided into four categories: ``random'', ``medium'', ``medium-replay'', and ``medium-expert''. The ``random'' category includes data from a randomly generated policy, ``medium'' includes experiences collected from an early-stopped policy trained by SAC~\cite{SAC} for 1M steps, ``medium-replay'' includes a replay buffer of a policy trained to the performance level of the ``medium'' agent, and "medium-expert'' is a combination of the ``medium'' data mentioned above and ``expert'' data gathered by the well-trained SAC policy completion at a 50-50 ratio.

\begin{figure}[!htb]
    \centering
    \subfigure[HalfCheetah]{
    \label{fig:halfcheetah}
    \includegraphics[scale=0.33]{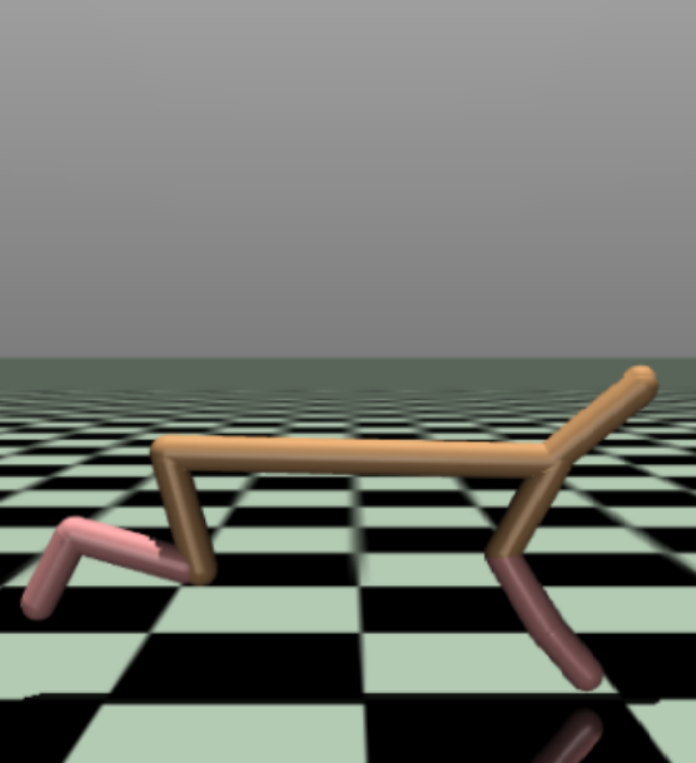}
    }
    \subfigure[Hopper]{
    \label{fig:hopper}
    \includegraphics[scale=0.33]{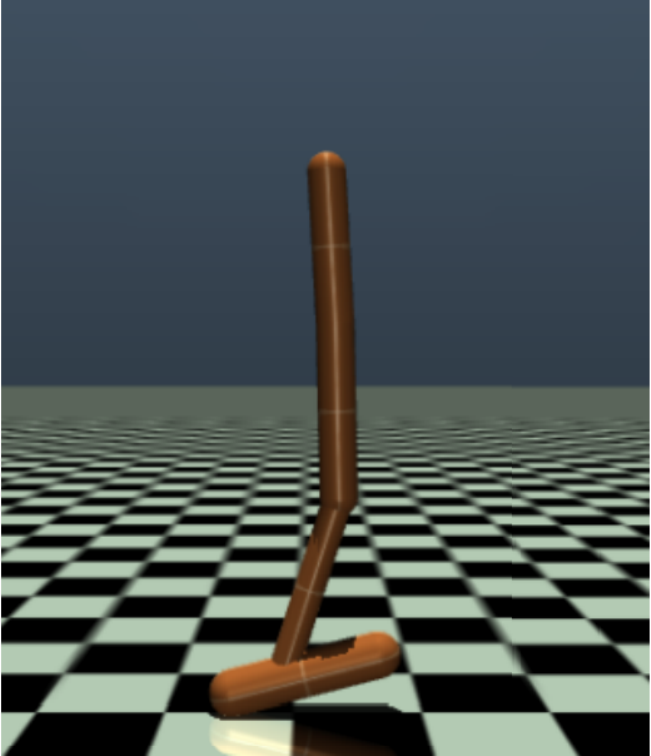}
    }
    \subfigure[Walker2d]{
    \label{fig:walker2d}
    \includegraphics[scale=0.33]{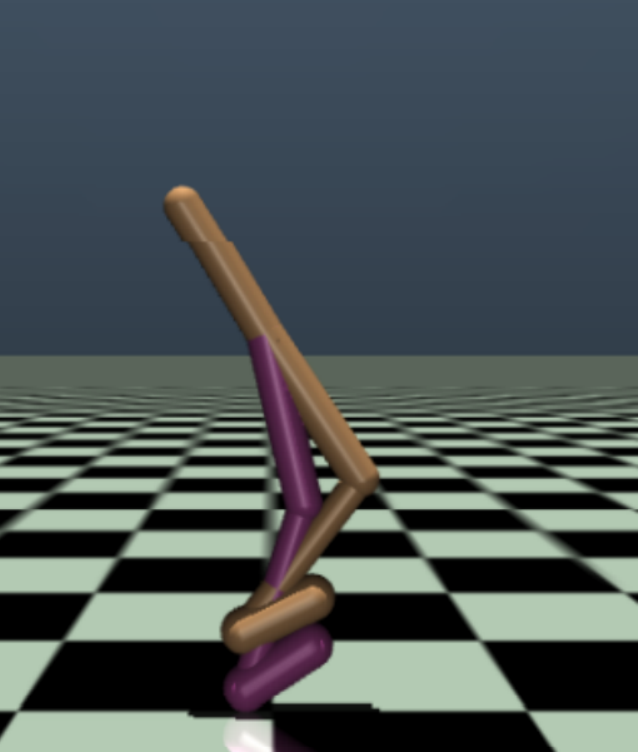}
    }
    \caption{Halfcheetah, hopper, and walker2d tasks.}
    \label{fig:mujocoviaual}
\end{figure}

The D4RL framework suggests using a normalized score metric to evaluate the performance of offline RL algorithms. This metric compares the expected return of an RL algorithm to that of a random policy and an expert policy. The expected return of a random policy on the dataset is represented as $C_r$ (reference minimum score), and the expected return of an expert policy is represented as $C_e$ (reference maximum score). When an offline RL algorithm's expected return is $C$, after being trained on a specific dataset, the normalized score $\tilde{C}$ can be calculated by
\begin{align}
    \label{eq:normalizedscore}
    \tilde{C} & = \frac{C - C_r}{C_e - C_r}\times 100 \nonumber \\ &= \frac{C - \mathrm{performance}\ \mathrm{of}\ \mathrm{random}\ \mathrm{policy}}{\mathrm{performance}\ \mathrm{of}\ \mathrm{expert}\ \mathrm{policy} - \mathrm{performance}\ \mathrm{of}\ \mathrm{random}\ \mathrm{policy}} \times 100.
\end{align}

The normalized score metric is used to evaluate the performance of offline RL algorithms and ranges roughly from 0 to 100. A score of 0 corresponds to the performance of a random policy, and 100 corresponds to the performance of an expert policy. The detailed reference minimum and maximum scores, represented as $C_r$ and $C_e$ respectively, are provided in Table \ref{tab:referencescore}. It is worth noting that the same reference minimum and maximum scores are applied across all tasks, regardless of the type of dataset.

\begin{table}[h]
	\centering
	\begin{tabular}{c|c|c}
		\toprule
		Task Name   & Reference min score $C_r$ & Reference max score $C_e$ \\
		\midrule
	    Halfcheetah & $-$280.18 & 12135.0 \\
	    Hopper & $-$20.27 & 3234.3 \\
		Walker2d & 1.63 & 4592.3 \\
		\bottomrule
	\end{tabular}
	\caption{The referenced min score and max score for the D4RL dataset.}
		\label{tab:referencescore}
\end{table}

\subsection{Tasks requiring policy to generalize}
\label{supple-exp-generalization}
The ``Halfcheetah-jump'' dataset proposed in~\cite{MOPO} requires the agent to simultaneously run and jump to achieve maximum height, utilizing a pre-existing offline training dataset of halfcheetah running.
Specially, the maximum velocity are reset from 1 to be 3.
The behavior policy is trained with the reward function 
\begin{equation}
	\label{equ:reward-original}
	r(s, a) = \max\{v_x, 3\} - 0.1*\|a\|_2^2,
\end{equation}
where $v_x$ denotes the velocity along the x-axis.
We use SAC to train the behavior policy for 1M steps with learning rate of $1e-5$ and use the entire training replay buffer as the trajectories for the offline dataset.
Then the new reward function is 
\begin{equation}
	\label{equ:reward-mopo}
	r(s, a) = \max\{v_x, 3\} - 0.1*\|a\|_2^2 + 15*z,
\end{equation}
where $z$ denotes the z-position of the halfcheetah.

\begin{figure}[h]
	\centering
	\subfigure[Halfcheetah with dangerous states.]{
		\label{fig:unhealthy}
		\includegraphics[scale=0.203]{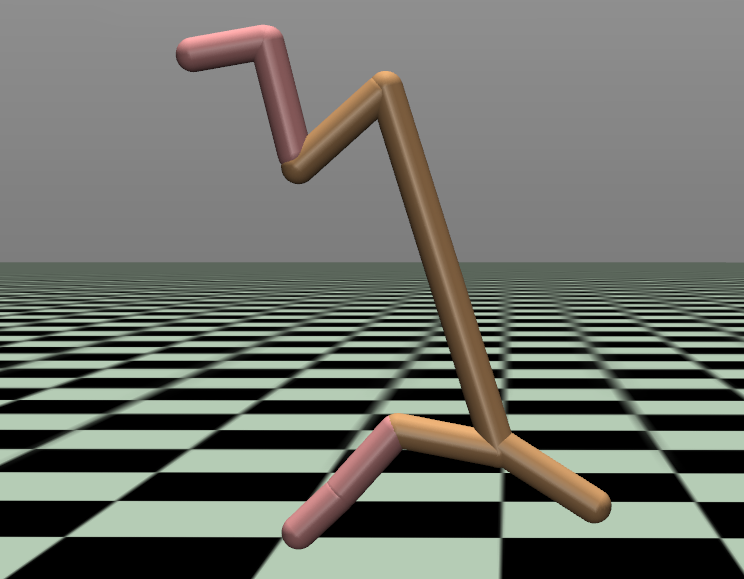}
	}
	\subfigure[Halfcheetah jumping healthily.]{
		\label{fig:healthy}
		\includegraphics[scale=0.2]{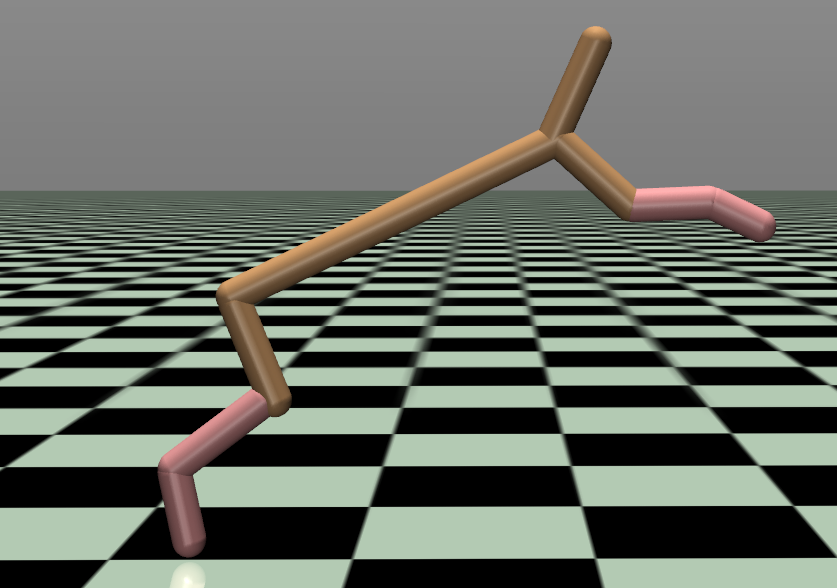}
	}
	\caption{Tasks requiring halfcheetah to generalize to jump.}
	\label{fig:halfcheetah-jump}
\end{figure}

We collect the ``Halfcheetah-jump-hard'' dataset based on ``Halfcheetah-jump''.
Firstly, we collect the whole offline dataset with a random policy that generates actions from the uniform distribution in the action space for 1M steps.
Secondly, we observe that when incentivize the halfcheetach to jump with the reward function defined in Equation~\ref{equ:reward-mopo},  halfcheetach can learn two kind of strategies, as shown in Figure~\ref{fig:halfcheetah-jump}.
To punish the halfcheetach with unhealthy dangerous states shown in Figure~\ref{fig:unhealthy} and only incentivize the halfcheetach to jump healthily as Figure~\ref{fig:healthy}, we modify the new reward function in the ``halfcheetah-jump-hard'' task to:
\begin{equation}
	r(s, a)  = \left\{
	\begin{array}{lcl}
		\max\{v_x, 3\} - 0.1*\|a\|_2^2 + 15*z,  &\text{if $y<1$,} \\ 
		{-3,} &\text{\ \  \  \  \ \  \ \text{if $y \ge 1$,} \ \ \ \ \ \ } 
	\end{array}  
	\right.
\end {equation}
where $y$ is the angle of the front tip of halfcheetach.

\subsection{Hyperparameters}
\label{supple-hyperparameters}
We adopt most hyper-parameters from optimized MOPO implementations~\cite{lu2022revisiting}, except for setting the number of ensembles $N=7$, the rollout length $h^p=1$ for walker2d tasks and $h^p=5$ for others.
We list the hyper-parameters of our implemented MOPO and ORPO in  Table 2.

\begin{table}[h]
	\centering
	\small
	 \resizebox{0.85\textwidth}{!}{
		\begin{tabular}{l|l|c|c|c}
			\toprule
			\textbf{\!\!\!Dataset type\!\!} & \textbf{Environment} &\makecell[c]{Optimized MOPO~\cite{lu2022revisiting} \\ $(N, \lambda^p,\text{type},h^p)$}  &\makecell[c]{MOPO (Our implementation) \\ $N=7,( \lambda^p,\text{type},h^p)$}& \makecell[c]{ORPO \\ $ (\lambda^o, h^o$)}\\ \midrule
			\!\!\!random & halfcheetah &10, 6.64,  Ensemble Std, 12&6.64,  Ensemble Std, 5&0.015, 5\\
			\!\!\!random & hopper & 6,  4.46, Max Aleatoric, 47 &\makecell[c]{ 5.9, Max Aleatoric, 5}&0.015, 5\\
			\!\!\!random & walker2d & 10, 0.21, Ensemble Var, 12  &0.21, Ensemble Var, 1 &0.002, 5 \\
			\!\!\!medium & halfcheetah & 12, 5.92, Ensemble Var, 6&5.92, Ensemble Var, 5&0.015, 1\\
			\!\!\!medium & hopper & 7, 5.0, Ensemble Std, 42  & \makecell[c]{5.0, Max Aleatoric, 5} &0.015, 1 \\
			\!\!\!medium & walker2d & 8, 5.28, Ensemble Std, 14  &5.28, Ensemble Std, 1&0.015, 1\\
			\!\!\!medium-replay &halfcheetah  & 11, 0.96, Ensemble Var, 37 &  0.96, Ensemble Var, 5& 0.002, 5\\
			\!\!\!medium-replay & hopper & 7, 5.9, Max Aleatoric, 5 & 2.5, Max Aleatoric, 5&0.005, 5\\
			\!\!\!medium-replay & walker2d & 13, 2.48, Ensemble Std, 47  &2.48, Ensemble Std, 1&0.015, 1\\
			\!\!\!med-expert\!\!\! & halfcheetah & 7, 4.56, Max Aleatoric, 5   &\makecell[c]{4.56, Max Aleatoric, 5} & 0.015, 1\\
		\!\!\!med-expert\!\!\! & hopper & 12, 39.08, Max Aleatoric, 43  &10, Max Aleatoric, 5 &0.015, 5 \\
		\!\!\!med-expert\!\!\! & walker2d & 12, 0.99, Ensemble Std, 37 &2.5, Max Aleatoric, 1&0.015, 1 \\
		\bottomrule
	\end{tabular}
}
	\caption{Tuned hyper-parameters used in the D4RL ``v2''  datasets.}
	\label{fig:d4rl-tuned-hyperparams}
	\normalsize
\end{table}

We set $\lambda^o=0.015$ by default for 9 out of the 12 datasets used in the D4RL benchmarks, expect for ``Walker2d-random-v2'', ``Halfcheetah-medium-replay-v2'' and ``Hopper-medium-replay-v2'' datasets.
We denote the rollout length of $\pi^o$ as $h^o$ and search over $h^o\in \{1,5\}$.
For ``Halfcheetah-jump'' and ``Halfcheetah-jump-hard'' datasets, we adopt the hyper-parameters used in MOPO, setting $\lambda^p=1$ and $h^p=5$. Further, we set $\lambda^o=0.1$ and $h^o=5$ for ORPO on these two datasets.

As a Dyna-style method, we update the rollout policy using the composite dataset $ \D \cup \ooD$.
Following MOPO, we sample a batch of $256$ transitions, $5\%$ of them from the replay buffer $\D$ with optimistic rollouts and the rest of them from $\oD$.
While for the pessimistic offline policy optimization, we utilize the composite datasets $\D \cup \poD \cup \ppD$.
We also sample a batch of $256$ transitions.
For the proportion of the three, we adopt $0.5:0.3:0.2$ on all ``medium-replay'' datasets and $0.05:0.45:0.5$ for others.
Besides, on  ``Halfcheetah-jump'' and ``Halfcheetah-jump-hard'' datasets, we set the proportion to $0.05:0.05:0.9$.

\subsection{Compute Infrastructure}
 The experiments are carried out on one NVIDIA GTX 3090 Ti GPU with Intel i9-12900K CPU and 64 GB memory.
 
\quad
\section{Extensive Experiments Results}

\subsection{Baselines}
\label{exp-all-baselines}
We compare ORPO with several state-of-the-art algorithms, including 
1) CQL~\cite{CQL}, a conservative Q-learning algorithm based on SAC~\cite{SAC} that minimizes Q-values of OOD actions.
2) TD3+BC~\cite{TD3BC}, a model-free algorithm that adopts adaptive BC constraint to regularize the policy in training.
3) MBPO~\cite{MBPO}, a model-based algorithm for online RL.
4) MOPO~\cite{MOPO}, a model-based algorithm based on MBPO, which further penalizes rewards by uncertainty.
5) MOPO (TD3+BC), our baseline that replaces SAC in MOPO with TD3+BC for policy optimization.
6) ORPO (SAC), same to our practical instantiation of ORPO except for utilizing SAC for pessimistic offline policy optimization.
7) OROO, an algorithm that  derived from MOPO by directly replacing the P-MDP with the O-MDP.
Our implementation of these baseline algorithms refers to their public source codes.
The implementation of CQL, TD3+BC is referred to \url{https://github.com/takuseno/d3rlpy/}.
In particular, we utilize the re-implementation of MOPO on \url{https://github.com/yihaosun1124/OfflineRL-Kit} due to its practical performance.

\subsection{Additional Ablation Study}
\label{appendix:baselines}

\begin{table}[h]\centering
	\addtolength{\tabcolsep}{-3pt}
	\renewcommand\arraystretch{0.8}
	\small
	\centering
	\resizebox{\linewidth}{!}{
		\begin{tabular}{c@{\hspace{3pt}}l@{\hspace{-0.5pt}}r@{\hspace{-1pt}}lr@{\hspace{-0.5pt}}lr@{\hspace{-0.5pt}}lr@{\hspace{-0.5pt}}lr@{\hspace{-0.5pt}}lr@{\hspace{-0.5pt}}lr@{\hspace{-0.5pt}}lr@{\hspace{-0.5pt}}lr@{\hspace{-0.5pt}}lr@{\hspace{-0.5pt}}lr@{\hspace{-0.5pt}}lr@{\hspace{-0.5pt}}lr}
			\toprule
			\multicolumn{1}{l}{}     &                                      & \multicolumn{4}{c|}{Model-free methods}  & \multicolumn{8}{c}{Model-based methods}  & \\
			\midrule
			\multicolumn{1}{l}{}     &                                      & \multicolumn{2}{c}{\makecell[c]{CQL \\ (2020 NeurIPS)}}  & \multicolumn{2}{c}{\makecell[c]{TD3+BC \\ (2021 NeurIPS)}}   & 
			\multicolumn{2}{c}{\makecell[c]{MBPO \\  (2019 NeurIPS)}}& \multicolumn{2}{c}{\makecell[c]{MOPO \\ (2020 NeurIPS)}} & 
			\multicolumn{2}{c}{\makecell[c]{ORPO \\ (Ours)}}&
			\multicolumn{2}{c}{\makecell[c]{ORPO (SAC) \\ (Baseline)}}&
			\multicolumn{2}{c}{\makecell[c]{OROO \\ (Baseline)}}& 
			\multicolumn{2}{c}{\makecell[c]{MOPO (TD3+BC) \\ (Baseline)}}\\
			\midrule
			\multirow{3}{*}{\rotatebox[origin=c]{90}{Random}} & HalfCheetah     & \colorbox{white}{27.0} & {\color[HTML]{525252} $\pm$0.6}& \colorbox{white}{11.3} & {\color[HTML]{525252} $\pm$0.5}  & \colorbox{white}{38.5} & {\color[HTML]{525252} $\pm$1.3}& \colorbox{white}{20.7}  & {\color[HTML]{525252} $\pm$1.8} &  \textbf{40.8} & {\color[HTML]{525252} $\pm$1.6} & \colorbox{white}{9.2} & {\color[HTML]{525252} $\pm$0.2} & \colorbox{white}{40.7} & {\color[HTML]{525252} $\pm$1.8} & \colorbox{white}{29.0} & {\color[HTML]{525252} $\pm$1.6}   \\
			
			& Hopper         & \colorbox{white}{16.2} & {\color[HTML]{525252} $\pm$2.5}  & \colorbox{white}{12.7} & {\color[HTML]{525252} $\pm$3.9}  &  \colorbox{white}{13.4} & {\color[HTML]{525252} $\pm$9.0} & \textbf{31.7}  & {\color[HTML]{525252} $\pm$0.3}    & \colorbox{white}{9.2} & {\color[HTML]{525252} $\pm$ 1.4} & \colorbox{white}{9.6} & {\color[HTML]{525252} $\pm$0.0} & \colorbox{white}{14.8} & {\color[HTML]{525252} $\pm$9.0} & \colorbox{white}{22.2} & {\color[HTML]{525252} $\pm$12.0} \\
			
			& Walker2d        & \colorbox{white}{1.2} & {\color[HTML]{525252} $\pm$0.5}  & \colorbox{white}{2.1} & {\color[HTML]{525252} $\pm$1.2}   & \colorbox{white}{1.2} & {\color[HTML]{525252} $\pm$1.6} & \colorbox{white}{1.7}  & {\color[HTML]{525252} $\pm$0.5}    & \textbf{10.8} & {\color[HTML]{525252} $\pm$9.3} & \colorbox{white}{0.1} & {\color[HTML]{525252} $\pm$0.0} & \colorbox{white}{3.2} & {\color[HTML]{525252} $\pm$1.4} & \colorbox{white}{2.6} & {\color[HTML]{525252} $\pm$1.6} \\
			\midrule
			\multirow{3}{*}{\rotatebox[origin=c]{90}{Medium}} & HalfCheetah      & \colorbox{white}{52.6} & {\color[HTML]{525252} $\pm$0.3} & \colorbox{white}{48.4} & {\color[HTML]{525252} $\pm$0.3}    &  \colorbox{white}{64.1} & {\color[HTML]{525252} $\pm$ 10.3}&
			\colorbox{white}{71.1}  & {\color[HTML]{525252} $\pm$2.6}   &  \textbf{73.4} & {\color[HTML]{525252} $\pm$0.5} & \colorbox{white}{68.3} & {\color[HTML]{525252} $\pm$0.3} & \colorbox{white}{69.5} & {\color[HTML]{525252} $\pm$ 9.3}& \colorbox{white}{68.1} & {\color[HTML]{525252} $\pm$1.5}  \\
			
			& Hopper       & \textbf{78.9} & {\color[HTML]{525252} $\pm$6.4}  & \colorbox{white}{56.4} & {\color[HTML]{525252} $\pm$4.9}  &  \colorbox{white}{1.0} & {\color[HTML]{525252} $\pm$0.4} &
			\colorbox{white}{20.7}  & {\color[HTML]{525252} $\pm$12.9}   & \colorbox{white}{30.4} & {\color[HTML]{525252} $\pm$37.4}   & \colorbox{white}{2.7} & {\color[HTML]{525252} $\pm$0.0}  & \colorbox{white}{1.2} & {\color[HTML]{525252} $\pm$0.7} & \colorbox{white}{33.4} & {\color[HTML]{525252} $\pm$32.3}\\
			
			& Walker2d        & \colorbox{white}{82.2} & {\color[HTML]{525252} $\pm$2.6}  & \textbf{80.8} & {\color[HTML]{525252} $\pm$2.9}  & \colorbox{white}{7.3} & {\color[HTML]{525252} $\pm$13.0} &
			\colorbox{white}{16.8}  & {\color[HTML]{525252} $\pm$15.0}  & \colorbox{white}{55.5} & {\color[HTML]{525252} $\pm$23.4} & \colorbox{white}{8.1} & {\color[HTML]{525252} $\pm$0.8} & \colorbox{white}{6.9} & {\color[HTML]{525252} $\pm$11.6}& \colorbox{white}{42.5} & {\color[HTML]{525252} $\pm$19.3}   \\
			
			\midrule
			\multirow{3}{*}{\rotatebox[origin=c]{90}{\shortstack{Medium\\Replay}}} & HalfCheetah      & \colorbox{white}{49.5} & {\color[HTML]{525252} $\pm$0.5}  & \colorbox{white}{44.2} & {\color[HTML]{525252} $\pm$0.5}  & 
			\colorbox{white}{48.1} & {\color[HTML]{525252} $\pm$ 24.5}&
			\colorbox{white}{62.5}  & {\color[HTML]{525252} $\pm$10.4}    &
			\textbf{72.8} & {\color[HTML]{525252} $\pm$0.9} & \colorbox{white}{57.5} & {\color[HTML]{525252} $\pm$3.0} &
			\colorbox{white}{41.4} & {\color[HTML]{525252} $\pm$27.3}& \colorbox{white}{65.6} & {\color[HTML]{525252} $\pm$1.8} \\
			
			& Hopper       & \colorbox{white}{99.2} & {\color[HTML]{525252} $\pm$1.6}  & \colorbox{white}{56.3} & {\color[HTML]{525252} $\pm$20.8}   & \colorbox{white}{98.7} & {\color[HTML]{525252} $\pm$9.2} &
			{100.8}  & {\color[HTML]{525252} $\pm$4.9}  & \textbf{104.6} & {\color[HTML]{525252} $\pm$1.5} & \colorbox{white}{51.8} & {\color[HTML]{525252} $\pm$0.2} & \colorbox{white}{92.7} & {\color[HTML]{525252} $\pm$20.0}& \colorbox{white}{88.4} & {\color[HTML]{525252} $\pm$29.7} \\
			
			& Walker2d       & \colorbox{white}{80.7} & {\color[HTML]{525252} $\pm$10.7} & \colorbox{white}{75.7} & {\color[HTML]{525252} $\pm$7.6}   &  \colorbox{white}{75.5} & {\color[HTML]{525252} $\pm$10.8}&
			\colorbox{white}{80.0}  & {\color[HTML]{525252} $\pm$8.9}   &  \textbf{91.1} & {\color[HTML]{525252} $\pm$2.0}& \colorbox{white}{54.2} & {\color[HTML]{525252} $\pm$26.9} & \colorbox{white}{68.0} & {\color[HTML]{525252} $\pm$17.5} & \colorbox{white}{73.3} & {\color[HTML]{525252} $\pm$13.2} \\
			
			\midrule
			\multirow{3}{*}{\rotatebox[origin=c]{90}{\shortstack{Medium\\Expert}}} & HalfCheetah      & \colorbox{white}{64.2} & {\color[HTML]{525252} $\pm$11.5}  & \colorbox{white}{86.0} & {\color[HTML]{525252} $\pm$6.7}  & \colorbox{white}{73.8} & {\color[HTML]{525252} $\pm$9.0}&
			\colorbox{white}{80.8}  & {\color[HTML]{525252} $\pm$11.4}      & \textbf{101.5} & {\color[HTML]{525252} $\pm$3.1}& \colorbox{white}{78.7} & {\color[HTML]{525252} $\pm$7.5} & \colorbox{white}{81.6} & {\color[HTML]{525252} $\pm$10.8}& \colorbox{white}{95.0} & {\color[HTML]{525252} $\pm$8.1}\\
			
			& Hopper          & \colorbox{white}{68.2} & {\color[HTML]{525252} $\pm$25.1} & \colorbox{white}{100.0} & {\color[HTML]{525252} $\pm$9.8}   & \colorbox{white}{5.9} & {\color[HTML]{525252} $\pm$6.7}&
			\colorbox{white}{21.1}  & {\color[HTML]{525252} $\pm$20.0} &  \textbf{111.0} & {\color[HTML]{525252} $\pm0.6$}  & \colorbox{white}{2.5} & {\color[HTML]{525252} $\pm$0.1} & \colorbox{white}{5.6} & {\color[HTML]{525252} $\pm$8.3}& \colorbox{white}{101.4} & {\color[HTML]{525252} $\pm$18.3}  \\
			
			& Walker2d      & \colorbox{white}{109.6} & {\color[HTML]{525252} $\pm$0.3}  & \textbf {110.3} & {\color[HTML]{525252} $\pm$0.5}  &  \colorbox{white}{24.1} & {\color[HTML]{525252} $\pm$37.9}&
			\colorbox{white}{102.1}  & {\color[HTML]{525252} $\pm$8.5}  &  {108.8} & {\color[HTML]{525252} $\pm$3.2} & \colorbox{white}{5.8} & {\color[HTML]{525252} $\pm$1.7} & \colorbox{white}{12.1} & {\color[HTML]{525252} $\pm$18.1}& \colorbox{white}{95.3} & {\color[HTML]{525252} $\pm$16.8}  \\
			
			
			\bottomrule
		\end{tabular}
	}
	\caption{Average normalized score and the standard deviation of all baseline methods with the `v2' dataset of D4RL. The highest-performing scores are highlighted. 
		We run all methods over 5 different seeds and take the average scores over the last 10 iterations of evaluation.
	} \label{table:all-results}
\end{table}

In this section, we provide more results on the D4RL benchmark that we omit in the main text due to space limitations.
We show the full performance comparison of ORPO against related methods as well as our baselines.

\subsubsection{The effect of offline RL for pessimistic policy optimization:}We set the baseline MOPO (TD3+BC) which replace SAC in MOPO with TD3+BC to migrate the effect of different RL algorithms on performance gain over MOPO.

\subsubsection{The effect of RL algorithms for pessimistic policy optimization:}
We study the effect of offline policy optimization algorithm on the performance.
In our practical instantiation of ORPO, we utilize the model-free offline algorithm, TD3+BC, for offline policy optimization.
We additionally set up a baseline, ORPO (SAC), utilizing the online RL algorithms, SAC to optimize the output policy. 
We summarize the results in Table~\ref{table:all-results}.
As shown, utilizing SAC for pessimistic offline policy optimization can not achieve satisfactory performance, which demonstrates that utilizing offline algorithms for policy optimization in ORPO is effective.
We believe future advances of offline RL algorithms can further improve the performance of our ORPO framework.

\subsubsection{The potential benefits of the O-MDP:}
In previous sections, we show that OROO is enough to improve performance in some datasets.
In summary, according to the results in Table~\ref{table:all-results}, we observe that OROO outperforms MBPO on 6 out of 12 datasets and MOPO on 3 out of 12 datasets, which indicates the potential benefits of the O-MDP in model-based offline RL.

\begin{figure}[h]
	\centering
	\includegraphics[scale=0.5]{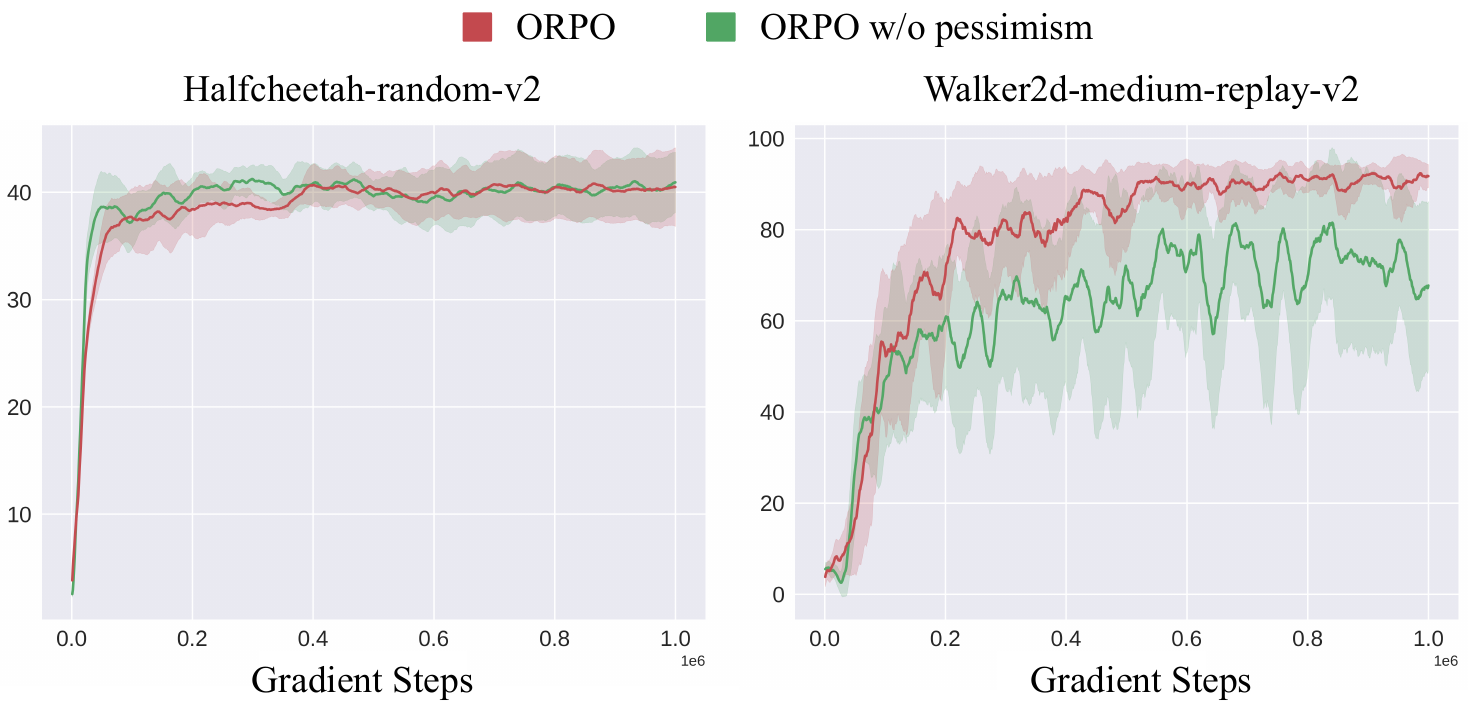}
	\caption{Learning curves of ORPO and the baseline optimizing output policies in the model MDP instead of the P-MDP.}
	\label{fig:pessimism}
\end{figure}

\subsubsection{The necessity of pessimism:}
Our further study isolates the effect of pessimism on the performance.
We set another baseline named ``ORPO w/o pessimism'', which replaces the P-MDP used in ORPO with the model MDP.
As shown in the Figure~\ref{fig:pessimism}, we observe that this baseline can match the performance of ORPO in the ``Halfcheetah-random-v2'' dataset.
Note that OROO can also perform well on this dataset, achieving state-of-the-art performance.
This indicates that optimism is effective for this dataset and there is no need to be pessimistic.
However, for various datasets where pessimism is important, such as  ``Walker2d-medium-replay-v2'', we can observe that ORPO outperforms the baseline without pessimism for a large margin.
We conclude that ORPO is robust to adapt to various datasets that either pessimism or optimism is more important.


 \subsection{Training Time of ORPO}
As we mentioned in Section~\ref{sec:conclusion}, one of the limitations of ORPO is that we need to train an rollout policy, which takes an additional time cost.
We train ORPO and various related baselines on two datasets and record the training time in Table ~\ref{tab:training_time} for quantification.
Since ORPO is an algorithmic framework, our training time depends on the selected algorithms to train the rollout policy and the output policy.
Note that our practical instantiation is based on MOPO and TD3+BC.
From the comparison of training time, we conclude that most time is spent during training the rollout policy and the output policy.
The training time of ORPO is approximately the sum of the time of TD3+BC and MOPO.
We also record the training time of the baseline with the well-trained optimistic rollout policy proposed in Section~\ref{sec:rollout}.
Results show that with pre-trained rollout policies, the additional time overhead compared to the P-MDP baseline is very little.

\begin{table}[h]
	\centering
		\resizebox{0.87\textwidth}{!}{
		\begin{tabular}{c|c|c| c | c | c}
			\toprule
			{Datasets}&TD3+BC&CQL &MOPO& ORPO  & \makecell[c]{ORPO (with trained rollout policy)  }   \\ 
			\midrule
			HalfCheetah-Random-v2 &3896&11945& 8054&13275&9869   \\
			Walker2d-Medium-replay-v2 &4241&11796&6686 &9334  &6715\\ 
			\bottomrule 
		\end{tabular}
		}
	\caption{The training time (seconds) of ORPO and various baselines.}
\label{tab:training_time}
\end{table}

\end{document}